\documentclass[letterpaper]{article} 
\usepackage{aaai23}  
\usepackage{times}  
\usepackage{helvet}  
\usepackage{courier}  
\usepackage[hyphens]{url}  
\usepackage{graphicx} 
\usepackage{natbib}  
\usepackage{caption} 
\usepackage{bbm}
\usepackage{wrapfig}
\usepackage[utf8]{inputenc} 
\usepackage[T1]{fontenc}    
\usepackage{hyperref}       
\usepackage{url}            
\usepackage{booktabs}       
\usepackage{amsfonts}       
\usepackage{nicefrac}       
\usepackage{microtype}      
\usepackage{xcolor}         
\usepackage[utf8]{inputenc} 
\usepackage[T1]{fontenc}    
\usepackage{url}            
\usepackage{booktabs}       
\usepackage{amsfonts}       
\usepackage{nicefrac}       
\usepackage{microtype}      
\usepackage{lipsum}		
\usepackage{graphicx}
\usepackage{natbib}
\usepackage{doi}
\usepackage{bbm}
\usepackage{caption}
\usepackage[american]{babel}
\usepackage{filecontents}
\usepackage{microtype}
\usepackage{subfigure}
\usepackage{booktabs} 
\usepackage[normalem]{ulem}
\usepackage{graphicx}
\usepackage{caption}
\usepackage{tikz}
\usetikzlibrary{arrows.meta, automata,
                positioning,
                quotes}
 \usepackage{amsmath, amssymb}
\usepackage{pst-node}

\usepackage{newfloat}
\usepackage{listings}

\usepackage{algorithmic}
\usepackage{algorithm}
\usepackage{wrapfig}
\usepackage{lipsum}
\usepackage{verbatim}
\usepackage{amsmath}
\allowdisplaybreaks
\usepackage{amsthm}
\usepackage{cleveref}
\usepackage{microtype}
\usepackage{color}
\usepackage{eucal}
\usepackage{epsfig,amsmath,amssymb,amsfonts,amstext,amsthm,mathrsfs}
\usepackage{latexsym,graphics,epsf,epsfig,psfrag}
\usepackage{dsfont,color,epstopdf,fixmath}
\usepackage{enumitem}
\usepackage{algorithm,algorithmic,latexsym}

\newtheorem{assumption}{\textbf{Assumption}}
\newtheorem{definition}{\textbf{Definition}}

\newtheorem{lemma}{\textbf{Lemma}}
\newtheorem{theorem}{\textbf{Theorem}}

\newtheorem{remark}{\textbf{Remark}}

\usepackage[acronym]{glossaries} 

\newcommand{\kp}{\mathsf P}
\newcommand{\cp}{\mathcal{P}}
\newcommand{\mcs}{\mathcal{S}}
\newcommand{\mca}{\mathcal{A}}
\newcommand{\nn}{\nonumber}
\newcommand{\mE}{\mathbb{E}}

\newcommand{\su}{\acute{s}}
\newcommand{\sm}{\grave{s}}
\DeclareCaptionStyle{ruled}{labelfont=normalfont,labelsep=colon,strut=off} 
\floatstyle{ruled}
\newfloat{listing}{tb}{lst}{}
\floatname{listing}{Listing}
\title{Robust Average-Reward Markov Decision Processes}
\author {
   Yue Wang,\textsuperscript{\rm 1}
    Alvaro Velasquez, \textsuperscript{\rm 2}
    George Atia, \textsuperscript{\rm 3}
    Ashley Prater-Bennette, \textsuperscript{\rm 4}
    Shaofeng Zou \textsuperscript{\rm 1}
}
\affiliations {
    \textsuperscript{\rm 1} University at Buffalo, The State University of New York\\
    \textsuperscript{\rm 2} University of Colorado Boulder \\
    \textsuperscript{\rm 3} University of Central Florida\\
    \textsuperscript{\rm 4} Air Force Research Laboratory\\
    ywang294@buffalo.com, alvaro.velasquez@colorado.edu, george.atia@ucf.edu,  ashley.prater-bennette@us.af.mil, szou3@buffalo.edu
}

\begin{document}

\maketitle

\begin{abstract}
In robust Markov decision processes (MDPs), the uncertainty in the transition kernel is addressed by finding a policy that optimizes the worst-case performance over an uncertainty set of MDPs. While much of the literature has focused on discounted MDPs, robust average-reward MDPs remain largely unexplored. In this paper, we focus on robust average-reward MDPs, where the goal is to find a policy that optimizes the worst-case average reward over an uncertainty set. We first take an approach that approximates average-reward MDPs using discounted MDPs. We prove that the robust discounted value function converges to the robust average-reward as the discount factor $\gamma$ goes to $1$, and moreover, when $\gamma$ is large, any optimal policy of the robust discounted MDP is also an optimal policy of the robust average-reward. We further design a robust dynamic programming approach, and theoretically characterize its convergence to the optimum. Then, we investigate robust average-reward MDPs directly without using discounted MDPs as an intermediate step. We derive the robust Bellman equation for robust average-reward MDPs, prove that the optimal policy can be derived from its solution, and further design a robust relative value iteration algorithm that provably find its solution, or equivalently, the optimal robust policy.

\end{abstract}

\section{Introduction}
A Markov decision process (MDP) is an effective mathematical tool for 
sequential decision-making in stochastic environments \cite{derman1970finite,puterman1994markov}. Solving an MDP problem entails finding an optimal policy that maximizes a cumulative reward according to a given criterion. 
However, in practice there could exist a mismatch between the assumed MDP model and the underlying environment due to various factors, such as non-stationarity of the environment, modeling error, exogenous perturbation, partial observability, and adversarial attacks. The ensuing model mismatch could result in solution policies with poor performance. 

This challenge spurred noteworthy efforts on developing and analyzing  a framework of robust MDPs e.g., \cite{bagnell2001solving,nilim2004robustness,iyengar2005robust}. 
Rather than adopting a fixed MDP model, in the robust MDP setting, one seeks to optimize the worst-case performance over an uncertainty set of possible MDP models.  
The solution to the robust MDP problem provides performance guarantee for all uncertain MDP models, and is thus robust to the model mismatch. 

Robust MDP problems falling under different reward optimality criteria are fundamentally different. 
In robust discounted  MDPs, the goal is to find a policy that maximizes the discounted cumulative reward in the worst case. In this setting, as the agent interacts with the environment, the reward 
received diminishes exponentially over time. 
Much of the prior work in the robust setting has focused on the discounted reward formulation. 
The model-based method, e.g., \citep{iyengar2005robust,nilim2004robustness,bagnell2001solving,satia1973markovian,wiesemann2013robust,tamar2014scaling,lim2019kernel,xu2010distributionally,yu2015distributionally,lim2013reinforcement}, where information about the uncertainty set is assumed to be known to the learner, unveiled several fundamental characterizations of robust discounted  MDPs. This was further extended to the more practical model-free setting in which only samples from a simulator (the centroid of the uncertainty set) are available to the learner. For example, the value-based method \cite{roy2017reinforcement,badrinath2021robust,wang2021online,tessler2019action,zhou2021finite,yang2021towards,panaganti2021sample,goyal2018robust,kaufman2013robust,ho2018fast,ho2021partial,si2020distributionally} optimizes the worst-case performance using the robust value function as an intermediate step; on the other hand, the model-free policy-based method \cite{russel2020robust,derman2021twice,eysenbach2021maximum,wang2022policy} directly optimizes the policy and is thus scalable to large/continuous state and action spaces.

Although discounted MDPs induce an elegant Bellman operator that is a contraction, and have been studied extensively, the policy obtained usually has poor long-term performance when a system operates for an extended period of time. When the discount factor is very close to $1$, the agent may prefer to compare policies on the basis of their average expected reward instead of their expected total discounted reward, e.g., queueing control, inventory management in supply chains, scheduling automatic guided vehicles and applications in communication networks \cite{kober2013reinforcement}. Therefore, it is also important to optimize the long-term average performance of a system.

However, robust MDPs under the average-reward criterion are largely understudied. Compared to the discounted setting, the average-reward setting depends on the limiting behavior of the underlying stochastic process, and hence is markedly more intricate. A recognized instance of such intricacy concerns the one-to-one correspondence between the stationary policies and the limit points of state-action frequencies, which while true for discounted MDPs, breaks down under the average-reward criterion even in the non-robust setting except in some very special cases \cite{puterman1994markov,atia2021steady}. This is largely due to dependence of the necessary conditions for establishing a contraction in average-reward settings on the graph structure of the MDP, versus the discounted-reward setting where it simply suffices to have a discount factor that is strictly less than one. Heretofore, only a handful of studies have considered average-reward MDPs in the robust setting. The first work by \cite{tewari2007bounded} considers robust average-reward MDPs under a specific finite interval uncertainty set, but their method is not easily applicable to other uncertainty sets. More recently, \cite{lim2013reinforcement} proposed an algorithm for robust average-reward MDPs under the $\ell_1$ uncertainty set. However, obtaining fundamental characterizations of the problem 
and convergence guarantee remains elusive. 
\subsection{Challenges and Contributions}
In this paper, we derive characterizations 
of robust average-reward MDPs with general uncertainty sets, and develop model-based approaches with provable theoretical guarantee. Our approach is fundamentally different from previous work on robust discounted MDPs, robust and non-robust average-reward MDPs. In particular, the key challenges and the main contributions are summarized below.

\begin{itemize}[leftmargin=*]
    \item \textbf{We characterize the limiting behavior of robust discounted  value function as the discount factor $\gamma\to 1$.} For the standard \emph{non-robust} setting and for a specific transition kernel, the discounted non-robust value function converges to the average-reward non-robust value function as $\gamma\to 1$ \cite{puterman1994markov}. However, in the robust setting, we need to consider the worst-case limiting behavior under all possible transition kernels in the uncertainty set. Hence, the previous point-wise convergence result \cite{puterman1994markov} cannot be directly applied. In \cite{tewari2007bounded}, a finite interval uncertainty set is studied, where due to its special structure, the number of possible worst-case transition kernels of robust discounted MDPs is finite, and hence the order of $\min$ (over transition kernel) and $\lim_{\gamma\to 1}$ can be exchanged, and therefore, the robust discounted value function converges to the robust average-reward value function. This result, however, does not hold for  general uncertainty sets investigated in this paper.
    We first prove the \textit{uniform} convergence of discounted non-robust value function to average-reward   w.r.t. the transition kernels and policies. Based on this uniform convergence, we show the convergence of the robust discounted value function to the robust average-reward. This uniform convergence result is the first in the literature and is of key importance to motivate our algorithm design and to guarantee convergence to the optimal robust policy in the average-reward setting.   
    
    \item \textbf{We design algorithms for robust policy evaluation and optimal control based on the limit method.} Based on the uniform convergence, we then use robust discounted  MDPs to approximate robust average-reward MDPs. 
    We show that when $\gamma$ is large, any optimal policy of the robust discounted MDP is also an optimal policy of the robust average-reward, and hence solves the robust optimal control problem in the average reward setting. This result is similar to the Blackwell optimality \cite{blackwell1962discrete,hordijk2002blackwell} for the non-robust setting, however, our proof is fundamentally different. Technically, the proof in \cite{blackwell1962discrete,hordijk2002blackwell} is based on the fact that the difference between the discounted value functions of two policies is a rational function of the discount factor, which has a finite number of zeros. However, in the robust setting with a general uncertainty set, the difference is no longer a rational function due to the min over the transition kernel. 
    We construct a novel proof based on the limiting behavior of robust discounted MDPs, and show that the (optimal) robust discounted value function converges to the (optimal) robust average-reward as $\gamma\to1$. Motivated by these insights, we then design our algorithms by applying a sequence of robust discounted Bellman operators while increasing the discount factor at a certain rate. We prove that our method can (i) evaluate the robust average-reward for a given policy and; (ii) find the optimal robust value function and, in turn, the optimal robust policy for general uncertainty sets. 
    
      \item \textbf{We design a robust relative value iteration method without using the discounted MDPs as an intermediate step.} 
      We further pursue a direct approach that solves the robust average-reward MDPs without using the limit method, i.e., without using discounted MDPs as an intermediate step. We derive a robust Bellman equation for robust average-reward MDPs, and show that the pair of robust relative value function and robust average-reward is a solution to the robust Bellman equation under the average-reward setting. We further prove that if we can find any solution to the robust Bellman equation, then the optimal policy can be derived by a greedy approach. The problem hence can be equivalently solved by solving the robust Bellman equation. We then design a robust value iteration method which provably converges to the solution of the robust Bellman equation, i.e., solve the optimal policy for the robust average-reward MDP problem.
\end{itemize}

\subsection{Related Work}

\noindent
\textbf{Robust discounted MDPs.}
 Model-based methods for robust discounted MDPs were studied in \cite{iyengar2005robust,nilim2004robustness,bagnell2001solving,satia1973markovian,wiesemann2013robust,lim2019kernel,xu2010distributionally,yu2015distributionally,lim2013reinforcement,tamar2014scaling}, where the uncertainty set is assumed to be known, and the problem can be solved using robust dynamic programming. Later, the studies were generalized to the model-free setting where stochastic samples from the centroid MDP of the uncertainty set are available in an online fashion \cite{roy2017reinforcement,badrinath2021robust,wang2021online,wang2022policy,tessler2019action} and an offline fashion \cite{zhou2021finite,yang2021towards,panaganti2021sample,goyal2018robust,kaufman2013robust,ho2018fast,ho2021partial,si2020distributionally}.  There are also empirical studies on robust RL, e.g., \cite{vinitsky2020robust,pinto2017robust,abdullah2019wasserstein,hou2020robust,rajeswaran2017epopt,huang2017adversarial,kos2017delving,lin2017tactics,pattanaik2018robust,mandlekar2017adversarially}.  For discounted MDPs, the robust Bellman operator is a contraction, based on which robust dynamic programming and value-based methods can be designed. In this paper, we focus on robust average-reward MDPs. However, the robust Bellman operator for average-reward MDPs is not a contraction, and its fixed point may not be unique. Moreover, the average-reward setting depends on the limiting
behavior of the underlying stochastic process, which is thus more intricate.


\noindent
\textbf{Robust average-reward MDPs.} Studies on robust average-reward MDPs are quite limited in the literature.
Robust average-reward MDPs under a specific finite interval uncertainty set was studied in \cite{tewari2007bounded}, where the authors showed the existence of a Blackwell optimal policy, i.e., there exists some $\delta\in[0,1)$, such that the optimal robust policy exists and remains unchanged for any discount factor $\gamma\in[\delta,1)$. However, this result depends on the structure of the uncertainty set. For general uncertainty sets, the existence of a Blackwell optimal policy may not be guaranteed. More recently, \cite{lim2013reinforcement} designed a model-free algorithm for a specific $\ell_1$-norm uncertainty set and characterized its regret bound. However, their method also relies on the structure of the $\ell_1$-norm uncertainty set, and may not be generalizable to other types of uncertainty sets. In this paper, our results can be applied to various types of uncertainty sets, and thus is more general.

\section{Preliminaries and Problem Model}
In this section, we introduce some preliminaries on discounted MDPs, average-reward MDPs, and robust MDPs.

\vspace{0.2cm}
\noindent\textbf{Discounted MDPs.}
A discounted MDP  $(\mathcal{S},\mathcal{A}, \mathsf P, r, \gamma)$ is specified by: a state space $\mcs$, an action space $\mca$, a transition kernel $\mathsf P=\left\{p^a_s \in \Delta(\mcs), a\in\mca, s\in\mcs\right\}$\footnote{$\Delta(\mcs)$: the $(|\mcs|-1)$-dimensional probability simplex on $\mcs$. }, where $p^a_s$ is the distribution of the next state over $\mcs$ upon taking action $a$ in state $s$ (with $p^a_{s,s'}$ denoting the probability of transitioning to $s'$), a reward function $r: \mcs\times\mca \to [0,1]$, and a discount factor $\gamma\in[0,1)$. At each time step $t$, the agent at state $s_t$ takes an action $a_t$, the  environment then transitions to the next state $s_{t+1}$ according to $p^{a_t}_{s_t}$, and produces a reward signal $r(s_t,a_t)\in [0,1]$ to the agent. In this paper, we also write $r_t=r(s_t,a_t)$ for convenience.

A stationary policy $\pi: \mcs\to \Delta(\mca)$ is a distribution over $\mca$ for any given state $s$, and the agent takes action $a$ at state $s$ with probability $\pi(a|s)$. The discounted value function of a stationary policy $\pi$ starting from $s\in\mcs$ is defined as the expected discounted cumulative reward by following policy $\pi$: $V^\pi_{\kp,\gamma}(s)\triangleq\mathbb{E}_{\pi, \kp}\left[\sum_{t=0}^{\infty}\gamma^t   r_t|S_0=s\right]$. 

\vspace{0.2cm}
\noindent\textbf{Average-Reward MDPs.}
Different from discounted MDPs,  average-reward MDPs do not discount the reward over time, and consider the 
behavior of the underlying Markov process under the steady-state distribution. More specifically, under a specific transition kernel $\kp$, the average-reward  of a policy $\pi$ starting from $s\in\mcs$ is defined as
\begin{align}
    g_\kp^\pi(s)\triangleq \lim_{n\to\infty} \mE_{\pi,\kp}\bigg[\frac{1}{n}\sum^{n-1}_{t=0} r_t|S_0=s \bigg],
\end{align}
which we also refer to in this paper as the average-reward value function for convenience. 

The average-reward value function can also be equivalently written as follows:
$
    g^\pi_\kp=\lim_{n\to\infty} \frac{1}{n}\sum^{n-1}_{t=0} (\kp^\pi)^t r_\pi\triangleq \kp^\pi_* r_\pi,
$
where $(\kp^\pi)_{s,s'}\triangleq \sum_a \pi(a|s) p^a_{s,s'}$ and $r_\pi(s)\triangleq\sum_a \pi(a|s)r(s,a)$ are the transition matrix and reward function induced by $\pi$, and $ \kp^\pi_*\triangleq\lim_{n\to\infty} \frac{1}{n}\sum^{n-1}_{t=0} (\kp^\pi)^t$ is the limit matrix of $\kp^\pi$. 

In the average-reward setting, we also define the following relative value function 
\begin{align}\label{eq:relativevaluefunction}
    V^\pi_\kp(s)\triangleq \mE_{\pi,\kp}\bigg[\sum^\infty_{t=0} (r_t-g^\pi_\kp)|S_0=s \bigg],
\end{align}
which is the cumulative difference over time between the reward and the average value $g^\pi_\kp$.
It has been shown that \cite{puterman1994markov}:
$
    V^\pi_\kp=H^\pi_\kp r_\pi,
$
where $H^\pi_\kp\triangleq (I-\kp^\pi+\kp^\pi_*)^{-1}(I-\kp^\pi_*)$ is defined as the deviation matrix of $\kp^\pi$. 

The relationship between the average-reward and the relative value functions 
can be characterized by the following Bellman equation \cite{puterman1994markov}:
\begin{align}
    V^\pi_\kp(s)=\mE_{\pi}\bigg[r(s,A)-g^\pi_\kp(s)+\sum_{s'\in\mcs} p^A_{s,s'}V^\pi_\kp (s') \bigg]. 
\end{align}

\vspace{0.2cm}
\noindent\textbf{Robust discounted and average-reward MDPs.}
For robust MDPs, the transition kernel is not fixed but belongs to some uncertainty set $\mathcal{P}$. After the agent takes an action, the environment transits to the next state according to an arbitrary transition kernel $\kp\in\cp$. In this paper, we focus on the $(s,a)$-rectangular uncertainty set \cite{nilim2004robustness,iyengar2005robust}, i.e., $\mathcal{P}=\bigotimes_{s,a} \mathcal{P}^a_s$, where $\mathcal{P}^a_s \subseteq \Delta(\mcs)$. We note that there are also studies on relaxing the $(s,a)$-rectangular uncertainty set to $s$-rectangular  uncertainty set, which is not the focus of this paper.

Under the robust setting, we consider the worst-case performance over the uncertainty set of MDPs. More specifically, the  robust discounted value function of a policy $\pi$ for a discounted MDP is defined as 
\begin{align}\label{eq:Vdef}
    V^\pi_{\cp,\gamma}(s)\triangleq \min_{\kappa\in\bigotimes_{t\geq 0} \mathcal{P}} \mathbb{E}_{\pi,\kappa}\left[\sum_{t=0}^{\infty}\gamma^t   r_t|S_0=s\right],
\end{align}
where $\kappa=(\mathsf P_0,\mathsf P_1...)\in\bigotimes_{t\geq 0} \mathcal{P}$.

In this paper, we focus on the following worst-case average-reward for a policy $\pi$: \begin{align}\label{eq:gppi}
    g^\pi_\cp(s)\triangleq \min_{\kappa\in\bigotimes_{t\geq 0} \mathcal{P}} \lim_{n\to\infty}\mathbb{E}_{\pi,\kappa}\left[\frac{1}{n}\sum_{t=0}^{n-1}r_t|S_0=s\right],
\end{align}
to which, for convenience, we refer as the robust average-reward value function.

For robust discounted MDPs, it has been shown that the robust discounted  value function is the unique fixed-point of the robust discounted  Bellman operator \cite{nilim2004robustness,iyengar2005robust,puterman1994markov}:
\begin{align}\label{eq:robustbellmanoperator}
    \mathbf T_\pi V(s)\triangleq \sum_{a\in\mca} \pi(a|s) \left(r(s,a)+\gamma \sigma_{\mathcal{P}^a_s}(V) \right),
\end{align}
where $\sigma_{\mathcal{P}^a_s}(V)\triangleq \min_{p\in\mathcal{P}^a_s} p^\top V$ is the support function of $V$ on $\mathcal{P}^a_s$. Based on the contraction of $\mathbf T_\pi$, robust dynamic programming approaches, e.g., robust value iteration, can be designed \cite{nilim2004robustness,iyengar2005robust} (see Appendix \ref{app:robustmdp} for a review of these methods).
However, there is no such contraction result for robust average-reward MDPs. In this paper, our goal is to find a policy that optimizes the robust average-reward value function:
\begin{align}\label{eq:objective}
    \max_{\pi\in\Pi} g^\pi_\cp(s), \text{ for any }s\in\mcs,
\end{align}
where $\Pi$ is the set of all stationary policies, and we denote
by $g^*_{\cp}(s)\triangleq \max_{\pi} g^\pi_{\cp}(s)$ the optimal robust average-reward. 

\section{Limit Approach for Robust Average-Reward MDPs}\label{sec:limit}
We first take a limit approach to solve the problem of robust average-reward MDPs in \cref{eq:objective}. It is known that under the non-robust setting, for any fixed $\pi$ and $\kp$, the discounted value function converges to the average-reward value function as the discount factor $\gamma$ approaches $1$ \cite{puterman1994markov}, i.e.,  
\begin{align}\label{eq:converge_nonrobust}
    \lim_{\gamma\to 1}(1-\gamma)V^\pi_{\kp,\gamma} = g^\pi_\kp.
\end{align}
We take a similar idea, and show that the same result holds in the robust case: $\lim_{\gamma\to 1}(1-\gamma)V^\pi_{\mathcal{P},\gamma}=g^\pi_{\mathcal{P}}$ under a mild assumption. Based on this result, we further design algorithms (Algorithms \ref{alg:evaluatoin} and \ref{alg:valueiteration}) that apply a sequence of robust discounted  Bellman operators while increasing the discount factor at a certain rate. We then theoretically prove that our algorithms converge to the optimal solutions.

In the following, we first show that the convergence $\lim_{\gamma\to 1}(1-\gamma)V^\pi_{\kp,\gamma} = g^\pi_\kp$ is uniform on the set $\Pi\times\cp$. 
In studies of average-reward MDPs, it is usually the case that a certain class of MDPs are considered, e.g., unichain and communicating \cite{wei2020model,zhang2021policy,chen2022learning,wan2021learning}. In this paper, we focus on the unichain setting to highlight the major technical novelty to achieve robustness.
\begin{assumption}\label{ass:compact}
For any $s\in\mcs,a\in\mca$, the uncertainty set $\cp^a_s$ is a compact subset of $\Delta(\mathcal{S})$. And for any $\pi\in\Pi, \kp\in\cp$, the induced MDP is a unichain.  
\end{assumption}
The first part of Assumption \ref{ass:compact} amounts to assuming that the uncertainty set is closed. We remark that many standard uncertainty sets satisfy this assumption, e.g., those defined by $\epsilon$-contamination \cite{hub65}, finite interval \cite{tewari2007bounded}, total-variation \cite{rahimian2022effective} and KL-divergence \cite{hu2013kullback}. The unichain assumption is also widely used in studies of average-reward MDPs, e.g., \cite{puterman1994markov,wan2021learning,zhang2021policy,lan2020first,zhang2021finite}. Also it is worth noting that under the unichain assumption, the robust average-reward is identical for every starting state, i.e., $g^\pi_\kp(s_1)=g^\pi_\kp(s_2), \forall s_1, s_2\in\mcs$ \cite{bertsekas2011dynamic}. 
\begin{remark}
    The results in this section actually only require the uniform boundedness of $\|H^\pi_\kp\|, \forall \pi\in\Pi,\kp\in\cp$ (\Cref{lemma:hbounded} in Appendix). Assumption \ref{ass:compact} is one sufficient condition.
\end{remark}

In \cite{puterman1994markov}, the convergence $\lim_{\gamma\to 1}(1-\gamma)V^\pi_{\kp,\gamma} = g^\pi_\kp$ for a fixed policy $\pi$ and a fixed transition kernel $\kp$  (non-robust setting)  is point-wise. However, such point-wise convergence does not provide any convergence guarantee on the robust discounted value function, as the robust value function measures the worst-case performance over the uncertainty set and the order of $\lim$ and $\min$ may not be exchanged in general. 
In the following theorem,  we prove the uniform convergence of the discounted value function under the foregoing assumption. 
\begin{theorem}[Uniform convergence]\label{thm:uniform}
Under Assumption \ref{ass:compact}, the discounted value function converges uniformly to the average-reward value function  on $\Pi\times\mathcal{P}$ as $\gamma \to 1$, i.e.,  
\begin{align}
    \lim_{\gamma\to 1}(1-\gamma)V^\pi_{\kp,\gamma}=g^\pi_{\kp}, \text{ uniformly.}
\end{align}
 
\end{theorem}


With uniform convergence in Theorem \ref{thm:uniform}, the order of the limit $\gamma\to1$ and  $\min_{\mathsf P}$ can be interchanged, then the following convergence of the robust discounted value function can be established.
\begin{theorem}\label{thm:lim of robust}
The robust discounted  value function in \cref{eq:Vdef} converges to the robust average-reward uniformly on $\Pi$:
\begin{align}
    \lim_{\gamma\to 1}(1-\gamma)V^\pi_{\mathcal{P},\gamma}=g^\pi_{\mathcal{P}} \text{ uniformly.}
\end{align}

\end{theorem}
We note that a similar convergence result is shown in \cite{tewari2007bounded}, but only for a special uncertainty set of finite interval. Our \Cref{thm:lim of robust} holds for general compact uncertainty sets. Moreover, it is worth highlighting that our proof technique is fundamentally different from the one in \cite{tewari2007bounded}. Specifically, under the finite interval uncertainty set, the worst-case transition kernels are from a finite set, i.e., $V^\pi_{\cp,\gamma}=\min_{\kp\in\mathcal{M}}V^\pi_{\kp,\gamma}$ for a finite set $\mathcal{M} \subseteq \cp$. This hence implies the interchangeability of $\lim$ and $\min$. However, for general uncertainty sets, the number of worst-case transition kernels may not be finite. We demonstrate the interchangeability via our uniform convergence result in \Cref{thm:uniform}.

The previous two convergence results play a fundamental role in limit method for robust average-reward MDPs, and are of key importance to motivate the design of the following two algorithms, the basic idea of which is to apply a sequence of robust discounted Bellman operators on an arbitrary initialization while increasing the discount factor at a certain rate.

We first consider the robust policy evaluation problem, which aims to estimate the robust average-reward $g^\pi_\cp$ for a fixed policy $\pi$. This problem for robust discounted MDPs is well studied in the literature, however, results for robust average-reward MDPs are quite limited except for the one in \cite{tewari2007bounded} for a specific finite interval uncertainty set. We present the a robust value iteration (robust VI) algorithm for evaluating the robust average-reward with general uncertainty sets in Algorithm \ref{alg:evaluatoin}.
%
%
\vskip -0.1in
\begin{algorithm}[htb]
\caption{Robust VI: Policy Evaluation}
\label{alg:evaluatoin}
\textbf{Input}: $\pi, V_0(s)=0,\forall s ,T$ 
\begin{algorithmic}[1] 
\FOR{$t=0,1,...,T-1$}
\STATE{$\gamma_t\leftarrow \frac{t+1}{t+2}$}
\FOR{all $s\in\mathcal{S}$}
\STATE {$V_{t+1}(s)\leftarrow \mathbb{E}_{\pi}[(1-\gamma_t)r(s,A)+\gamma_t \sigma_{\mathcal{P}^A_s}(V_t)]$}
\ENDFOR
\ENDFOR
\STATE \textbf{return} $V_{T}$
\end{algorithmic}
\end{algorithm}
\vskip -0.1in
At each time step $t$, the discount factor $\gamma_t$ is set to $\frac{t+1}{t+2}$, which converges to $1$ as $t\to\infty$. Subsequently, a robust Bellman operator w.r.t discount factor $\gamma_t$ is applied on the current estimate $V_t$ of the robust discounted  value function $(1-\gamma_t) V^\pi_{\cp,\gamma_t}$. As the discount factor approaches $1$, the estimated robust discounted  value function converges to the robust average-reward $g^\pi_\cp$ by Theorem \ref{thm:lim of robust}. The following result shows that the output of Algorithm \ref{alg:evaluatoin} converges to the robust average-reward.
\begin{theorem}\label{thm:evaluation}
Algorithm \ref{alg:evaluatoin} converges to robust average reward, i.e., $\lim_{T\to\infty}V_T=g^\pi_{\mathcal{P}}$.
\end{theorem}

Besides the robust policy evaluation problem, it is also of great practical importance to find an optimal policy that maximizes the worst-case average-reward,  
i.e., to solve \cref{eq:objective}.
Based on a similar idea as the one of Algorithm \ref{alg:evaluatoin}, we extend our limit approach to solve the robust optimal control problem in Algorithm \ref{alg:valueiteration}.
\vskip -0.1in
\begin{algorithm}[H]
\caption{Robust VI: Optimal Control}
\label{alg:valueiteration}
\textbf{Input}: $V_0(s)=0,\forall s , T$ 
\begin{algorithmic}[1] 
\FOR{$t=0,1,...,T-1$}
\STATE{$\gamma_t\leftarrow \frac{t+1}{t+2}$}
\FOR{all $s\in\mathcal{S}$}
\STATE {$V_{t+1}(s)\leftarrow \underset{a\in\mca}{\max}\left\{(1-\gamma_t)r(s,a)+\gamma_t \sigma_{\mathcal{P}^a_s}(V_t)\right\}$}
\ENDFOR
\ENDFOR
\FOR{$s\in\mcs$}
\STATE $\pi_T(s)\leftarrow \arg\max_{a\in\mca}\left\{(1-\gamma_t)r(s,a)+\gamma_t \sigma_{\mathcal{P}^a_s}(V_T)\right\}$
\ENDFOR
\STATE \textbf{return} $V_{T},\pi_T$
\end{algorithmic}
\end{algorithm}
\vskip -0.1in

Similar to Algorithm \ref{alg:evaluatoin}, at each time step, the discount factor $\gamma_t$ is set to be closer to $1$, and a one-step robust discounted Bellman operator (for optimal control) w.r.t. $\gamma_t$ is applied to the current estimate $V_t$. The following theorem establishes that $V_T$ in Algorithm \ref{alg:valueiteration} converges to the optimal robust value function, hence can find the optimal robust policy.  
\begin{theorem}
\label{thm:val_converges}
The output $V_T$ in Algorithm \ref{alg:valueiteration} converges to the optimal robust average-reward $g^*_\cp$: $V_T \rightarrow g^*_\cp$ as $T\to\infty$.
\end{theorem} 


As discussed in \cite{blackwell1962discrete,hordijk2002blackwell}, the average-reward criterion is insensitive and under selective since it is only interested in the performance under the steady-state distribution. For example, two policies providing rewards: $100+0+0+\cdots$ and $0+0+0+\cdots$ are equally good/bad. Towards this issue, for the non-robust setting, a more sensitive term of optimality was introduced by Blackwell \cite{blackwell1962discrete}. More specifically, a policy is said to be Blackwell optimal if it optimizes the discounted value function for all discount factor $\gamma\in(\delta,1)$ for some $\delta\in(0,1)$. Together with \cref{eq:converge_nonrobust}, the optimal policy obtained by taking $\gamma\rightarrow 1$ is optimal not only for the average-reward criterion, but also for the discounted criterion with large $\gamma$. Intuitively, it is optimal under the average-reward setting, and is sensitive to early rewards. 

Following a similar idea, we justify that the obtained policy from \Cref{alg:valueiteration} is not only optimal in the robust average-reward setting, but also sensitive to early rewards.


Denote by $\Pi^*_D$ the set of all the deterministic optimal policies for robust average-reward (proved to exist in \Cref{lemma:deter opt policy}), i.e.
$
    \Pi_D^*=\left\{\pi\in\Pi_D: g^\pi_\cp=g^*_\cp \right\}. 
$
\begin{theorem}[Blackwell optimality]\label{thm:blackwell}
There exists $0<\delta<1$, such that for any $\gamma>\delta$, the deterministic optimal robust policy for robust discounted value function $V^*_{\cp,\gamma}$ belongs to $\Pi^*_D$. 
Moreover, when $\Pi^*_D$ is a singleton, there exists a unique Blackwell optimal policy. 
\end{theorem}
This result implies that using the limit method in this section to find the optimal robust policy for average-reward MDPs has an additional advantage that the policy it finds not only optimizes the average reward in steady state, but also is sensitive to early rewards. 


It is worth highlighting the distinction of our results from the technique used in the proof of Blackwell optimality 
\cite{blackwell1962discrete}.
%
%
In the non-robust setting, the existence of a stationary Blackwell optimal policy is proved via contradiction, where a difference function of two policies $\pi$ and $\nu$: $f_{\pi,\nu}(\gamma)\triangleq V^\pi_{\kp,\gamma}-V^\mu_{\kp,\gamma}$ is used in the proof. It was shown by contradiction that $f$ has infinitely many zeros, which however contradicts with the fact that $f$ is a rational function of $\gamma$ with a finite number of zeros. A similar technique was also used in \cite{tewari2007bounded} for the finite interval uncertainty set. Specifically, 
in \cite{tewari2007bounded}, it was shown that the worst-case transition kernels for any $\pi,\gamma$ are from a finite set $\mathcal{M}$, hence $f_{\pi,\nu}(\gamma)\triangleq \min_{\kp\in\mathcal{M} }V^\pi_{\kp,\gamma}-\min_{\kp\in\mathcal{M}}V^\mu_{\kp,\gamma}$ can also be shown to be a rational function with a finite number of zeroes. 
For a general uncertainty set $\mathcal{P}$, the difference function $f_{\pi,\nu}(\gamma)$, however, may not be rational. 
This makes the method in \cite{blackwell1962discrete,tewari2007bounded} inapplicable to our problem. 

\section{Direct Approach for Robust Average-Reward  MDPs}\label{sec:direct}
The limit approach in Section \ref{sec:limit} is based on the uniform convergence of the discounted value function, and uses discounted MDPs to approximate average-reward MDPs. 
In this section, we develop a direct approach to solving the robust average-reward MDPs that does not adopt discounted MDPs as intermediate steps.

For average-reward MDPs, the relative value iteration (RVI) approach \cite{puterman1994markov} is commonly used since it is numerically stable and has convergence guarantee.  
In the following, we generalize the RVI algorithm to the robust setting, and design the robust RVI algorithm in Algorithm \ref{alg:relative}.

We first generalize the relative value function in \cref{eq:relativevaluefunction} to the robust relative value function. The  robust relative value function measures the  difference between the worst-case cumulative reward and the worst-case average-reward for a policy $\pi$. 
\begin{definition}
The robust relative value function is defined as
\begin{align}
     V^\pi_\cp(s)\triangleq\min_{\kappa\in\bigotimes_{t\geq 0}\cp}\mathbb{E}_{\kappa,\pi}\bigg[\sum^\infty_{t=0}(r_t-g^\pi_\cp)|S_0=s\bigg],
\end{align}
where $g^\pi_\cp$ is the worst-case average-reward defined in \cref{eq:gppi}.
\end{definition}

The following theorem presents a robust Bellman equation for robust average-reward MDPs. 
%
%
%
\begin{theorem}\label{thm:bellman}
For any $s$ and $\pi$, $(V^\pi_\cp,g^\pi_\cp)$ is a solution to the following robust Bellman equation:
\begin{align}\label{eq:bellman}
    V(s)+g=\sum_a \pi(a|s)\left(r(s,a)+ \sigma_{\cp^a_s}(V)\right). 
\end{align}
\end{theorem}
It can be seen that the robust Bellman equation for average-reward MDPs has a similar structure to the one for discounted MDPs in \cref{eq:robustbellmanoperator} except for a discount factor. 
This actually reveals a fundamental difference between the robust Bellman operator of the  discounted  MDPs and the average-reward ones. For a discounted MDP, its robust Bellman operator is a contraction with constant $\gamma$ \cite{nilim2004robustness,iyengar2005robust}, and hence the fixed point is unique. Based on this, the robust value function can be found by recursively applying the robust Bellman operator (see Appendix \ref{app:robustmdp}). In sharp contrast, in the average-reward setting, the robust Bellman is not necessarily a contraction, and the fixed point may not be unique. Therefore, repeatedly applying the robust Bellman operator in the average-reward setting may not even converge,
which underscores that the two problem settings are fundamentally different.

We first derive the following equivalent optimality condition for robust average-reward MDPs. 
\begin{theorem}
\label{thm:opt_eq}
For any $(g,V)$ that is a solution to  
\begin{align}\label{eq:opt eq}
   \max_{a}\left\{r(s,a)-g+\sigma_{\cp^a_s}(V)-V(s) \right\}=0, \forall s,
\end{align}
$
    g=g^*_\cp.
$
If we further set
\begin{align}\label{eq:112}
\pi^*(s)=\arg\max_a \left\{ r(s,a)+\sigma_{\cp^a_s}(V)\right\}
\end{align}
for any $s\in\mcs$, 
then $\pi^*$ is an optimal robust policy. 
\end{theorem}
\Cref{thm:opt_eq} suggests that as long as we find a solution $(g,V)$ to \cref{eq:opt eq}, which though may not be unique, then $g$ is the optimal robust average-reward $g^*_\cp$, and the greedy policy $\pi^*$ is the optimal policy to our robust average-reward MDP problem in \cref{eq:objective}. 



In the following, we generalize the RVI approach to the robust setting, and design a robust RVI algorithm in \Cref{alg:relative}. We will further show that the output of this algorithm converges to a solution to \cref{eq:opt eq}, and further the optimal policy could be obtained by \cref{eq:112}.
\begin{algorithm}[htb]
\caption{Robust RVI}
\label{alg:relative}
\textbf{Input}: $V_0$, $\epsilon$ and arbitrary $s^*\in \mcs$
\begin{algorithmic}[1] 
\STATE {$w_0\leftarrow V_0-V_0(s^*)\mathbbm 1$}
\WHILE{$sp(w_t-w_{t+1})\geq \epsilon$}
\FOR{all $s\in\mathcal{S}$}
\STATE {$V_{t+1}(s)\leftarrow \max_a (r(s,a)+\sigma_{\cp^a_s}(w_t))$}
\STATE {$w_{t+1}(s)\leftarrow V_{t+1}(s)-V_{t+1}(s^*)$}
\ENDFOR
\ENDWHILE
\STATE \textbf{return} $w_{t}, V_t$
\end{algorithmic}
\end{algorithm}
Here $\mathbbm 1$ denotes the all-ones vector, and $sp$ denotes the span semi-norm: $sp(w)=\max_s w(s) -\min_s w(s)$.
Different from Algorithm \ref{alg:valueiteration}, in Algorithm \ref{alg:relative}, we do not need to apply the robust discounted  Bellman operator. The method directly solves the robust optimal control problem for average-reward robust MDPs. 


To study the convergence of the robust RVI algorithm, we first make an additional assumption as follows. 
\begin{assumption}\label{ass:unichain}
There exists a positive integer $J$ such that for any $\kp=\left\{p^a_{s}\in\Delta(\mcs) \right\}\in\cp$ and any stationary deterministic policy $\pi$, there exists $\kappa>0$ and a state $s\in\mcs$, such that  $((\kp^\pi)^J)_{x,s}\geq \kappa, \forall x\in\mcs$.
\end{assumption}
This assumption is shown to be equivalent to assuming unichain and aperiodic \cite{bertsekas2011dynamic}. It can be also replaced using some weaker ones, e.g., Proposition 4.3.2 of \cite{bertsekas2011dynamic}, or be removed by designing a variant of RVI, e.g., Proposition 4.3.4 of \cite{bertsekas2011dynamic}. 
In the following theorem, we show that our \Cref{alg:relative} converges to a solution of \cref{eq:opt eq}, hence according to \Cref{thm:opt_eq} if we set $\pi$ according to \eqref{eq:112},
then $\pi$ is the optimal robust policy.
\begin{theorem}\label{thm:conv of vi}
$(w_t,V_t)$ converges to a solution $(w,V)$ to \cref{eq:opt eq} as $\epsilon\to 0$. 
\end{theorem}


\begin{remark}
    In this section, we mainly present the robust RVI algorithm for the robust optimal control problem, and its convergence and optimality guarantee. A robust RVI algorithm for robust policy evaluation can be similarly designed by replacing the $\max$ in line 4, \Cref{alg:relative} with an expectation w.r.t. $\pi$. The convergence results in \Cref{thm:conv of vi} can also be similarly derived. 
\end{remark}

\section{Examples and Numerical Results}
In this section, we study  several commonly used uncertainty set models, including contamination model, Kullback-Lerbler (KL) divergence and total-variation defined model.

As can be observed from \Cref{alg:evaluatoin,alg:valueiteration,alg:relative}, for different uncertainty sets, the only difference lies in how the support function $\sigma_{\cp^a_s}(V)$ is calculated. In the sequel, we discuss how to efficiently calculate the support function for various uncertainty sets. 

We numerically compare our robust (relative) value iteration methods v.s. non-robust (relative) value iteration method on different uncertainty sets. Our experiments are based on the Garnet problem $\mathcal{G}(20,40)$ \cite{archibald1995generation}. More specifically, there are $20$ states and $30$ actions; the nominal transition kernel $\kp=\left\{ p^a_s\in\Delta(\mcs)\right\}$ is randomly generated according to the uniform distribution, and the reward functions $r(s,a)\sim\mathcal{N}(0,\sigma_{s,a})$, where $\sigma_{s,a}\sim \text{Uniform}[0,1]$. In our experiments, the uncertainty sets are designed to be centered at the nominal transition kernel. We run different algorithms, i.e., (robust) value iteration and (robust) relative value iteration, and obtain the greedy policies at each time step. Then, we use robust average-reward policy evaluation (Algorithm \ref{alg:evaluatoin}) to evaluate the robust average-reward of these policies. We plot the robust average-reward against the number of iterations.

\noindent\textbf{Contamination model.} For any $(s,a)$ the uncertainty set $\cp_s^a$ is defined as $\cp_s^a=\left\{q: q=(1-R)p_s^a+Rp', p'\in\Delta(\mcs) \right\}$, where $p_s^a$ is the nominal transition kernel. It can be viewed as an adversarial model, where at each time-step, the environment  transits according to the nominal transition kernel $p$ with probability $1-R$, and according to an arbitrary kernel $p'$ with probability $R$. Note that
$
    \sigma_{\cp^a_s}(V)=(1-R)(p_s^a)^\top V+R\min_s V(s).
$
Our experimental results under the contamination model are shown in \Cref{fig:contamination}.
\vskip -0.24in
\begin{figure}[ht]
\begin{center}
\subfigure[Robust VI.]{
\includegraphics[width=0.47\linewidth]{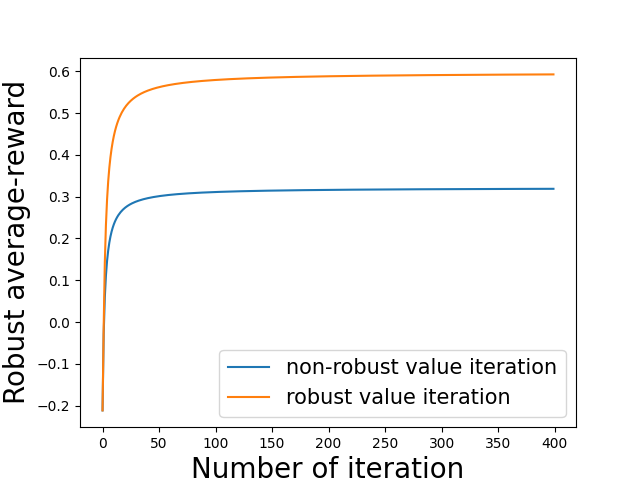}}
\subfigure[Robust RVI.]{
\includegraphics[width=0.47\linewidth]{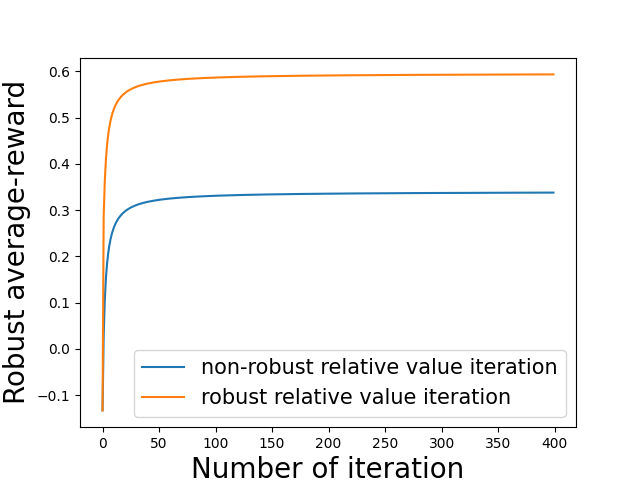}}
\vspace{-.25cm}
\caption{Comparison on  contamination model with $R=0.4$.}
\label{fig:contamination}
\end{center}
\vskip -0.2in
\end{figure}

\noindent\textbf{Total variation.} The total variation distance is another commonly used distance metric to measure the difference between two distributions. For two distributions $p$ and $q$, it is defined as
$
    D_{TV}(p,q)=\frac{1}{2} \|p-q\|_1. 
$
Consider an uncertainty set defined via total variation: $\cp_s^a=\left\{q: D_{TV}(q||p_s^a)\leq R \right\}$.
Then, its support function can be efficiently solved as follows \cite{iyengar2005robust}: 
$
    \sigma_{\cp^a_s}(V)=
    p^\top V  -R\min_{\mu\geq 0}\left\{\max_s (V(s)-\mu(s)) -\min_s(V(s)-\mu(s))\right\}.
$


Our experimental results under the total variation model are shown in \Cref{fig:tt}. 
\begin{figure}[ht]
\begin{center}
\subfigure[Robust VI.]{
\includegraphics[width=0.47\linewidth]{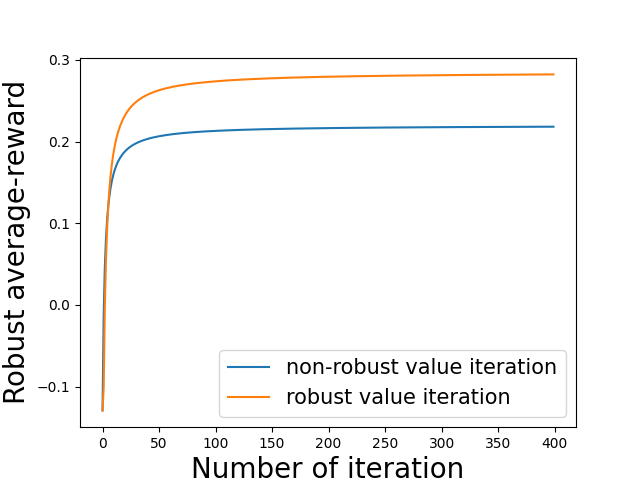}}
\subfigure[Robust RVI.]{
\includegraphics[width=0.47\linewidth]{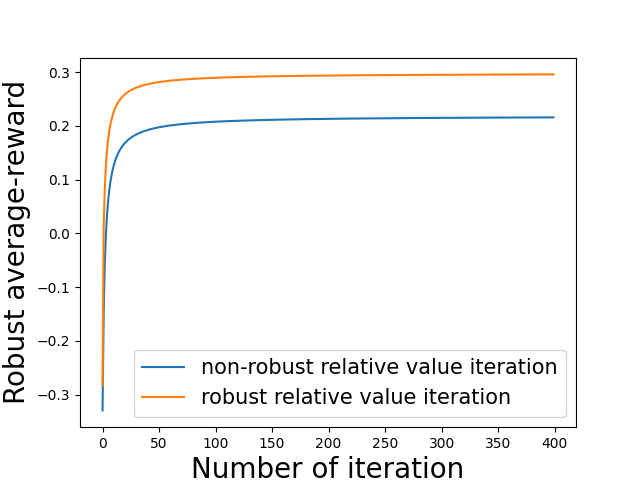}}
\caption{Comparison on total variation model with $R=0.6$.}
\label{fig:tt}
\end{center}
\vskip -0.2in
\end{figure}

\noindent\textbf{Kullback-Lerbler (KL) divergence.} 
The Kullback–Leibler divergence is widely used 
to measure the distance between two probability distributions. 
For distributions $p,q$, it is defined as 
$
    D_{KL}(q||p)=\sum_s q(s)\log\frac{q(s)}{p(s)}. 
$
Consider an uncertainty set defined via KL divergence: $\cp_s^a=\left\{q: D_{KL}(q||p_s^a)\leq R \right\}$. Then, its support function can be efficiently solved using the duality result in \cite{hu2013kullback}:
$\sigma_{\cp^a_s}(V)=
    -\min_{\alpha\geq 0} \left\{R\alpha+\alpha\log \left( p^\top e^{\frac{-V}{\alpha}}\right) \right\}. $
Our experimental results under the KL-divergence model are shown in \Cref{fig:kl}.
\vskip -0.24in
\begin{figure}[ht]
\begin{center}
\subfigure[Robust VI.]{
\includegraphics[width=0.47\linewidth]{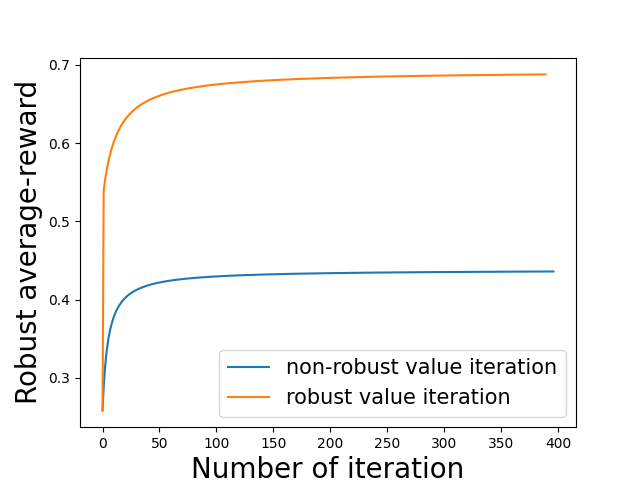}}
\subfigure[Robust RVI.]{
\includegraphics[width=0.47\linewidth]{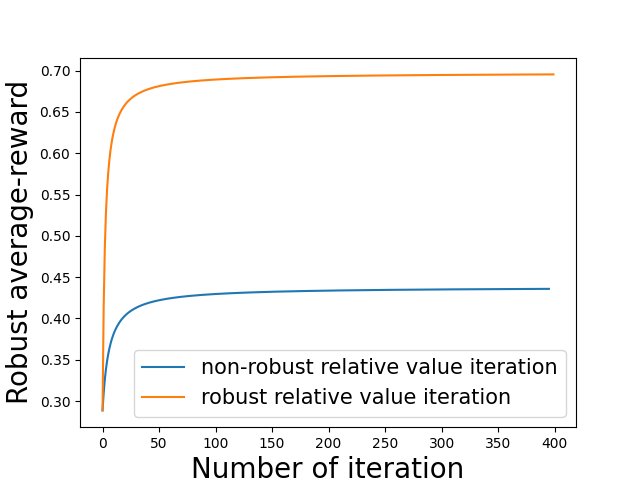}}
\caption{Comparison on KL-divergence model with $R=0.8$.}
\label{fig:kl}
\end{center}
\vskip -0.2in
\end{figure}



It can be seen that our robust methods can obtain policies that achieve higher worst-case reward. Also, both our limit-based robust value iteration and our direct method of robust relative value iteration converge to the optimal robust policies, which validates our theoretical results.

\section{Conclusion}
In this paper, we investigated the problem of robust MDPs under the average-reward setting. We established \textit{uniform} convergence of the discounted value function to average-reward, which further implies the uniform convergence of the \textit{robust} discounted value function to \textit{robust} average-reward. Based on this insight, we designed a robust dynamic programming approach using the robust discounted MDPs as an approximation (the limit method). We theoretically proved their convergence and optimality and proved a robust version of the Blackwell optimality \cite{blackwell1962discrete}. 
We then designed a direct approach for robust average-reward MDPs, where we derived the robust Bellman equation for robust average-reward MDPs.
We further designed a robust RVI method, which was proven to converge to the optimal robust solution. Technically, our proof techniques are fundamentally different from existing studies on average-reward robust MDPs, e.g., those in \cite{blackwell1962discrete,tewari2007bounded}.


\section{Acknowledgment}
This work was supported by the National Science Foundation under Grants CCF-2106560, CCF-2007783, CCF-2106339 and CCF-1552497. 

\newpage
\bibliography{aaai23}

\onecolumn
\appendix
\section{Review of Robust Discounted MDPs}\label{app:robustmdp}
In this section, we provide a brief review on the existing methods and results for robust discounted MDPs.

\subsection{Robust Policy Evaluation}
We first consider the robust policy evaluation problem, where we aim to estimate the robust value function $V^\pi_{\cp,\gamma}$ for any policy $\pi$. 
It has been shown that the robust Bellman operator $\mathbf T_\pi$ is a $\gamma$-contraction, and the robust value function $V^\pi_{\cp,\gamma}$ is its unique fixed-point. Hence by recursively applying the robust Bellman operator, we can find the robust discounted value function \cite{nilim2004robustness,iyengar2005robust}. 
\begin{algorithm}[H]
\caption{Policy evaluation for robust discounted MDPs}
\label{alg:discount evaluatoin}
\textbf{Input}: $\pi, V_0,T$ 
\begin{algorithmic}[1] 
\FOR{$t=0,1,...,T-1$}
\FOR{all $s\in\mathcal{S}$}
\STATE {$V_{t+1}(s)\leftarrow \mathbb{E}_{\pi}[r(s,A)+\gamma \sigma_{\mathcal{P}^A_s}(V_t)]$}
\ENDFOR
\ENDFOR
\STATE \textbf{return} $V_{T}$
\end{algorithmic}
\end{algorithm}

\subsection{Robust Optimal Control}
Another important problem in robust MDP is to find the optimal policy which maximizes the robust discounted value function:
\begin{align}
    \pi^*=\arg\max_{\pi} V^\pi_{\cp,\gamma}. 
\end{align}
A robust value iteration approach is developed in \cite{nilim2004robustness,iyengar2005robust} as follows. 
\begin{algorithm}[H]
\caption{Optimal Control for robust discounted MDPs}
\label{alg:discount value}
\textbf{Input}: $ V_0,T$ 
\begin{algorithmic}[1] 
\FOR{$t=0,1,...,T-1$}
\FOR{all $s\in\mathcal{S}$}
\STATE {$V_{t+1}(s)\leftarrow \max_a\left\{ r(s,a)+\gamma \sigma_{\mathcal{P}^a_s}(V_t)\right\}$}
\ENDFOR
\ENDFOR
\STATE {$\pi^*(s)\leftarrow \arg\max_a \left\{r(s,a)+\gamma\sigma_{\mathcal{P}^a_s}(V_T) \right\}, \forall s$}
\STATE \textbf{return} $\pi^*$
\end{algorithmic}
\end{algorithm}
\section{Equivalence between Time-Varying and Stationary Models}
We first provide an equivalence result between time-varying and stationary transition kernel models under stationary policies, which is an analog result to the one for robust discounted MDPs \cite{iyengar2005robust,nilim2004robustness}. This result will be used in our following proofs.

Recall the definitions of robust discounted value function and worst-case average reward in \cref{eq:Vdef,eq:gppi}, the worst-case is taken w.r.t. $\kappa=(\mathsf P_0,\mathsf P_1...)\in\bigotimes_{t\geq 0} \mathcal{P}$, therefore, the transition kernel at each time step could be different. This model is referred to as time-varying transition kernel model (as in \cite{iyengar2005robust,nilim2004robustness}). Another commonly used setting is that the transition kernels at different time step are the same, which is referred to as the stationary model \cite{iyengar2005robust,nilim2004robustness}. In this paper, we use the following notations to distinguish the two models. By $\mE_{\mathsf P}[\cdot]$, we denote the expectation when the transition kernels at all time steps are the same, $\mathsf P$, i.e., the stationary model. We also denote by $g^\pi_\kp(s)\triangleq \lim_{n\to\infty}\mathbb{E}_{\mathsf P,\pi}\left[\frac{1}{n}\sum_{t=0}^{n-1}r_t\big|S_0=s\right]$ and $V^\pi_{\kp,\gamma}(s)\triangleq \mathbb{E}_{\mathsf P,\pi}\left[ \sum_{t=0}^{\infty}\gamma^t r_t\big|S_0=s\right]$ being the expected average-reward and expected discounted value function under the stationary model $\kp$. By $\mE_{\kappa}[\cdot]$, we denote the expectation when the transition kernel at time $t$ is $\mathsf P_t$,  i.e., the time-varying model. 

For the discounted setting, it has been shown in \cite{nilim2004robustness} that for a stationary policy $\pi$,  any $\gamma\in[0,1)$, and any $s\in\mcs$,
\begin{align}\label{eq:equivalencediscounted}
        V^\pi_{\cp,\gamma}(s)&= \min_{\kappa\in\bigotimes_{t\geq 0} \mathcal{P}} \mathbb{E}_{\pi,\kappa}\left[\sum_{t=0}^{\infty}\gamma^t   r_t|S_0=s\right]\nn\\
        &=\min_{\mathsf P\in\mathcal{P}} \mathbb{E}_{\pi,\mathsf P}\left[\sum_{t=0}^{\infty}\gamma^t   r_t|S_0=s\right].
\end{align}
In the following theorem, we prove an analog of \cref{eq:equivalencediscounted} for robust-average reward MDPs that if we consider stationary policies, then the robust average-reward problem with the time-varying model can be equivalently solved by a stationary model.

Specifically, we define the worst-case average reward for the stationary transition kernel model as follows:
\begin{align}\label{eq:gppi1}
   \min_{\mathsf P\in \mathcal{P}} \lim_{n\to\infty}\mathbb{E}_{\pi,\mathsf P}\left[\frac{1}{n}\sum_{t=0}^{n-1}r_t\big|S_0=s\right].
\end{align}
Recall the worst-case average reward for the time-varying model in \cref{eq:gppi}. We will show that for any stationary policy, \cref{eq:gppi} can be equivalently solved by solving \cref{eq:gppi1}.
\begin{theorem}\label{thm:stationary}
Consider an arbitrary stationary policy $\pi$. Then, the worst-case average-reward under the time-varying model is the same as the one under the stationary model:
\begin{align}\label{eq:stationary g}
    g^\pi_\cp(s)&\triangleq \min_{\kappa\in\bigotimes_{t\geq 0} \mathcal{P}} \lim_{n\to\infty}\mathbb{E}_{\kappa,\pi}\left[\frac{1}{n}\sum_{t=0}^{n-1}r_t|S_0=s\right]\nn\\
    &=\min_{\mathsf P\in \mathcal{P}} \lim_{n\to\infty}\mathbb{E}_{\mathsf P,\pi}\left[\frac{1}{n}\sum_{t=0}^{n-1}r_t\big|S_0=s\right].
\end{align}
Similar result also holds for the robust relative value function:
\begin{align}\label{eq:stationary V}
        V^\pi_\cp(s)&\triangleq\min_{\kappa\in\bigotimes_{t\geq 0} \mathcal{P}}\mathbb{E}_{\kappa,\pi}\bigg[\sum^\infty_{t=0}(r_t-g^\pi_\cp)|S_0=s\bigg]\nn\\
        &=\min_{\kp\in\cp}\mathbb{E}_{\kp,\pi}\bigg[\sum^\infty_{t=0}(r_t-g^\pi_\cp)|S_0=s\bigg] . 
\end{align}
\end{theorem}
\begin{proof}
From the robust Bellman equation in \Cref{thm:bellman} \footnote{The proof of \Cref{thm:bellman} is independent of \cref{thm:stationary} and does not relay on the results to be showed here.}, we have that 
\begin{align}
    V^\pi_\cp(s)+g^\pi_\cp=\sum_a \pi(a|s)\left(r(s,a)+ \sigma_{\cp^a_s}(V^\pi_\cp)\right). 
\end{align}
Denote by $\arg\min_{p\in\cp^a_s} (p)^\top V^\pi_\cp\triangleq p^a_s$\footnote{We pick one  arbitrarily, if there are multiple minimizers.}, and denote by $\kp^\pi\triangleq \{ p^a_s : s\in\mcs,a\in\mca \}$. 
It then follows that
\begin{align}
    V^\pi_\cp(s)&=\sum_a \pi(a|s)\left(r(s,a)-g^\pi_\cp+ \sigma_{\cp^a_s}(V^\pi_\cp)\right)\nn\\
    &=\sum_a \pi(a|s)(r(s,a)-g^\pi_\cp)+ \sum_a \pi(a|s) \mE_{\kp^\pi}[V^\pi_\cp(S_1)|S_0=s,A_0=a]\nn\\
    &=\sum_a \pi(a|s)(r(s,a)-g^\pi_\cp)+  \mE_{\kp^\pi,\pi}[V^\pi_\cp(S_1)|S_0=s]\nn\\
    &=\sum_a \pi(a|s)(r(s,a)-g^\pi_\cp)+  \mE_{\kp^\pi,\pi}\bigg[\sum_a \pi(a|S_1) (r(S_1,a)-g^\pi_\cp)|S_0=s\bigg]+\mE_{\kp^\pi,\pi}\bigg[\sum_a \pi(a|S_1)\sigma_{\cp^a_{S_1}}(V^\pi_\cp)|S_0=s\bigg] \nn\\
    &=\sum_a\pi(a|s) (r(s,a)-g^\pi_\cp)+  \mE_{\kp^\pi,\pi}\left[r_1-g^\pi_\cp|S_0=s\right]+\mE_{\kp^\pi,\pi}\bigg[\sigma_{\cp^{A_1}_{S_1}}(V^\pi_\cp)|S_0=s\bigg] \nn\\
    &=\sum_a \pi(a|s)(r(s,a)-g^\pi_\cp)+  \mE_{\kp^\pi,\pi}\bigg[r_1-g^\pi_\cp\big|S_0=s\bigg]+\mE_{\kp^\pi,\pi}\bigg[(p^{A_1}_{S_1})^\top V^\pi_\cp|S_0=s\bigg] \nn\\
    &=\mE_{\kp^\pi,\pi}\bigg[r_0-g^\pi_\cp+ r_1-g^\pi_\cp|S_0=s\bigg]+\mE_{\kp^\pi,\pi}[V^\pi_\cp(S_2)|S_0=s] \nn\\
    &\quad\quad......\nn\\
    &=\mE_{\kp^\pi,\pi} \bigg[\sum^\infty_{t=0} (r_t-g^\pi_\cp)|s \bigg].
\end{align}
By the definition, the following always hold:
\begin{align}
        \min_{\kappa\in\bigotimes_{t\geq 0} \mathcal{P}}\mathbb{E}_{\kappa,\pi}\bigg[\sum^\infty_{t=0}(r_t-g^\pi_\cp)|S_0=s\bigg]\leq\min_{\kp\in\cp}\mathbb{E}_{\kp,\pi}\bigg[\sum^\infty_{t=0}(r_t-g^\pi_\cp)|S_0=s\bigg] . 
\end{align}
This hence implies that a stationary transition kernel sequence $\kappa=(\kp^\pi,\kp^\pi,...)$ is one of the worst-case transition kernels for $V^\pi_\cp$. Therefore, \cref{eq:stationary V} can be proved.  

Consider the transition kernel $\kp^\pi$. We denote its non-robust average-reward and the non-robust relative value function by $g^\pi_{\kp^\pi}$ and $V^\pi_{\kp^\pi}$. 
By the non-robust Bellman equation \cite{sutton2018reinforcement}, we have that 
\begin{align}
    V^\pi_{\kp^\pi}(s)=\sum_a \pi(a|s) (r(s,a)-g^\pi_{\kp^\pi})+\mE_{\kp^\pi,\pi}[V^\pi_{\kp^\pi}(S_1)|s].
\end{align}
On the other hand, the robust Bellman equation shows that 
\begin{align}
   V^\pi_\cp(s)=  V^\pi_{\kp^\pi}(s)=\sum_a \pi(a|s)(r(s,a)-g^\pi_\cp)+  \mE_{\kp^\pi,\pi}[V^\pi_{\kp^\pi}(S_1)|s].
\end{align}
These two equations hence implies that 
$
    g^\pi_\cp=g^\pi_{\kp^\pi},
$
and hence the stationary kernel $(\kp^\pi,\kp^\pi,...)$ is also a worst-case kernel of robust average-reward in the time-varying setting. This proves \cref{eq:stationary g}.
\end{proof}

\section{Proof of Theorem \ref{thm:uniform}}
In the proof, unless otherwise specified, we denote by $\|v\|$ the $l_\infty$ norm of a vector $v$, and for a matrix $A$, we denote by $\|A\|$ its matrix norm induced by $l_\infty$ norm, i.e., $\|A\|=\sup_{x\in\mathbb{R}^d}\frac{\|Ax\|_\infty}{\|x\|_\infty}$. 

\begin{lemma}\label{lemma:dis-ave}[Theorem 8.2.3 in \cite{puterman1994markov}]
For any $\kp$, $\gamma$, $\pi$, 
\begin{align}
    V^\pi_{\kp,\gamma}=\frac{1}{1-\gamma}g^\pi_\kp+h^\pi_\kp+f^\pi_\kp(\gamma),
\end{align}
where 
$h^\pi_\kp=H^\pi_\kp r_\pi$, and $f^\pi_\kp(\gamma)=\frac{1}{\gamma}\sum^\infty_{n=1} (-1)^n\left(\frac{1-\gamma}{\gamma} \right)^n (H^\pi_\kp)^{n+1} r_\pi$.
\end{lemma}

Following Proposition 8.4.6 in \cite{puterman1994markov}, we can show the following lemma.
\begin{lemma}\label{lemma:hbounded}
$H^\pi_\kp$ is continuous on $\Pi\times\mathcal{P}$. If $\Pi$ and $\mathcal{P}$ are compact, $\|H^\pi_\kp\|$ is uniformly bounded on $\Pi\times\mathcal{P}$, i.e., there exists a constant $h$, such that $\|H^\pi_\kp\|\leq h$ for any $\pi,\kp$. 
\end{lemma}

For simplicity, denote by 
\begin{align}
    S^\pi_\infty(\kp,\gamma)\triangleq\frac{1}{\gamma}\sum^\infty_{n=1} (-1)^n\left(\frac{1-\gamma}{\gamma} \right)^n (H^\pi_\kp)^{n+1} r_\pi,\nn\\
    S^\pi_N(\kp,\gamma)\triangleq\frac{1}{\gamma}\sum^N_{n=1} (-1)^n\left(\frac{1-\gamma}{\gamma} \right)^n (H^\pi_\kp)^{n+1} r_\pi.
\end{align}
Clearly $S^\pi_\infty(\kp,\gamma)=f^\pi_\kp(\gamma)$ and $\lim_{N\to\infty}S^\pi_N(\kp,\gamma)=S^\pi_\infty(\kp,\gamma)$ for any specific $\pi,\kp$. 

\begin{lemma}\label{lemma:Sn}
There exists $\delta\in (0,1)$, such that
\begin{align}
    \lim_{N\to\infty}S^\pi_N(\kp,\gamma)=S^\pi_\infty(\kp,\gamma)
\end{align}
uniformly on $\Pi\times \mathcal{P}\times [\delta,1]$. 
\end{lemma}
\begin{proof}
Note that $\|H^\pi_\kp\|\leq h$, hence there exists $\delta$, s.t. 
\begin{align}
    \frac{1-\delta}{\delta}h\leq k <1
\end{align}
for some constant $k$. 
Then for any $\gamma\geq\delta$, 
\begin{align}
    \frac{1-\gamma}{\gamma}h\leq  \frac{1-\delta}{\delta}h\leq k. 
\end{align}
Moreover, note that 
\begin{align}
    \left\|\frac{1}{\gamma}(-1)^n\left(\frac{1-\gamma}{\gamma} \right)^n (H^\pi_\kp)^{n+1} r\right\|\leq \frac{1}{\gamma}\left(\frac{1-\gamma}{\gamma} \right)^n h^{n+1}\leq \frac{hk^n}{\delta}\triangleq M_n,
\end{align}
which is because $\|A+B\|\leq \|A\|+\|B\|$ for induced $l_\infty$ norm, $\|Ax\|\leq \|A\|\|x\|$   and $\|r_\pi\|_\infty\leq 1$. 

Note that 
\begin{align}
    \sum^\infty_{n=1} M_n =\frac{h}{\delta}\frac{k}{1-k},
\end{align}
hence by Weierstrass $M$-test \cite{rudin2022functional}, $S^\pi_N(\kp,\gamma)$ uniformly converges to $S^\pi_\infty(\kp,\gamma)$ on $\Pi\times\mathcal{P}\times [\delta,1]$.  
\end{proof}

\begin{lemma}\label{lemma:Snlip}
There exists a uniform constant $L$, such that 
\begin{align}
    \|S^\pi_N(\kp,\gamma_1)-S^\pi_N(\kp,\gamma_2)\|\leq L|\gamma_1-\gamma_2|,
\end{align}
for any $N$, $\pi$,  $\kp,\gamma_1,\gamma_2\in[\delta,1]$.
\end{lemma}
\begin{proof}
We first show that $\gamma S^\pi_N(\kp,\gamma)=\sum^N_{n=1} (-1)^n\left(\frac{1-\gamma}{\gamma} \right)^n (H^\pi_\kp)^{n+1} r_\pi\triangleq T^\pi_N(\kp,\gamma)$ is uniformly Lipschitz w.r.t. the $l_\infty$ norm, i.e., 
\begin{align}
    \|T^\pi_N(\kp,\gamma_1)-T^\pi_N(\kp,\gamma_2)\|\leq l|\gamma_1-\gamma_2|,
\end{align}
for any $N$, $\pi$,  $\kp,\gamma_1,\gamma_2\in[\delta,1]$ and some constant $l$.

Clearly, it can be shown by verifying $\nabla T^\pi_N(\kp,\gamma)$ is uniformly bounded for any  $\pi,N,\kp$ or $\gamma$.

First, it can be shown that
\begin{align}
    \nabla T^\pi_N(\kp,\gamma)&=\sum^N_{n=1} (-1)^nn\left(\frac{1-\gamma}{\gamma} \right)^{n-1}\frac{-1}{\gamma^2} (H^\pi_\kp)^{n+1} r_\pi, 
\end{align}
and moreover
\begin{align}
    \|\nabla T^\pi_N(\kp,\gamma)\|&\leq \sum^N_{n=1} n\left(\frac{1-\gamma}{\gamma} \right)^{n-1}\frac{1}{\gamma^2} h^{n+1}\triangleq l_N(\gamma).
\end{align}

Note that 
\begin{align}
    h\frac{1-\gamma}{\gamma}l_N(\gamma)&=\sum^N_{n=1} n\left(\frac{1-\gamma}{\gamma} \right)^{n}\frac{1}{\gamma^2} h^{n+2},
\end{align}
then, we can show that 
\begin{align}
    &\left(1-h\frac{1-\gamma}{\gamma}\right)l_N(\gamma)\nn\\
    &=\sum^N_{n=1} n\left(\frac{1-\gamma}{\gamma} \right)^{n-1}\frac{1}{\gamma^2} h^{n+1}-\sum^N_{n=1} n\left(\frac{1-\gamma}{\gamma} \right)^{n}\frac{1}{\gamma^2} h^{n+2}\nn\\
    &=\frac{1}{\gamma^2}h^2-N\left(\frac{1-\gamma}{\gamma} \right)^{N}\frac{1}{\gamma^2}h^{N+2}+ \sum^{N}_{n=2}\left(\frac{1-\gamma}{\gamma} \right)^{n-1}\frac{1}{\gamma^2} h^{n+1}\nn\\
    &\leq \frac{1}{\gamma^2}h^2+ \frac{h^2}{\gamma^2}\frac{1-\gamma}{\gamma}h\frac{1}{1-\frac{1-\gamma}{\gamma}h}\nn\\
    &=\frac{h^2}{\gamma^2}+ \frac{h^2}{\gamma^2}\frac{1-\gamma}{\gamma}h\frac{1}{1-\frac{1-\gamma}{\gamma}h}.
\end{align}
Hence, we have that 
\begin{align}
    \|\nabla T^\pi_N(\kp,\gamma)\|&\leq l_N(\gamma)\leq \frac{1}{1-h\frac{1-\gamma}{\gamma}} \left(\frac{h^2}{\gamma^2}+ \frac{h^2}{\gamma^2}\frac{1-\gamma}{\gamma}h\frac{1}{1-\frac{1-\gamma}{\gamma}h}\right)\nn\\
    &\leq \frac{1}{1-k}\left(\frac{h^2}{\delta^2}+ \frac{h^2}{\delta^2}\frac{k}{1-k}\right), 
\end{align}
which implies a uniform bound on $\|\nabla T^\pi_N(\kp,\gamma)\|$. 

Now, we have that 
\begin{align}
    &|S^\pi_N(\kp,\gamma_1)-S^\pi_N(\kp,\gamma_2)|\nn\\
    &\leq\frac{|\gamma_2-\gamma_1|}{\gamma_1\gamma_2}\|T^\pi_N(\kp,\gamma_1)\|+\frac{\|T^\pi_N(\kp,\gamma_1)-T^\pi_N(\kp,\gamma_2)\|}{\gamma_2}.
\end{align}
To show $\|T^\pi_N(\kp,\gamma)\|$ is uniformly bounded, we have that
\begin{align}
    \|T^\pi_N(\kp,\gamma)\|&\leq \sum^N_{n=1}\left\|\left(\frac{1-\gamma}{\gamma} \right)^n (H^\pi_\kp)^{n+1} r\right\|\nn\\
    &\leq  \sum^N_{n=1}\left(\frac{1-\gamma}{\gamma} \right)^n h^{n+1}\nn\\
    &\leq \sum^N_{n=1} k^nh\nn\\
    &\leq h\frac{k}{1-k}.
\end{align}
Then, it follows that
\begin{align}
    &\|S^\pi_N(\kp,\gamma_1)-S^\pi_N(\kp,\gamma_2)\|\nn\\
    &=\bigg\|\frac{\gamma_2-\gamma_1}{\gamma_1\gamma_2}T^\pi_N(\kp,\gamma_1)+\frac{T^\pi_N(\kp,\gamma_1)-T^\pi_N(\kp,\gamma_2)}{\gamma_2}\bigg\|\nn\\
    &\leq \left( \frac{1}{\delta^2}h\frac{k}{1-k}+\frac{1}{\delta}\frac{1}{1-k}\left(\frac{h^2}{\delta^2}+ \frac{h^2}{\delta^2}\frac{k}{1-k}\right) \right) |\gamma_1-\gamma_2|\nn\\
    &\triangleq L|\gamma_1-\gamma_2|,
\end{align}
where $L=\left( \frac{1}{\delta^2}h\frac{k}{1-k}+\frac{1}{\delta}\frac{1}{1-k}\left(\frac{h^2}{\delta^2}+ \frac{h^2}{\delta^2}\frac{k}{1-k}\right) \right)$ is a universal constant that does not depend on $N,\kp,\pi$ or $\gamma$. 
\end{proof}

\begin{lemma}\label{lemma:lip_f}
$
    S^\pi_\infty(\kp,\gamma)
$ uniformly converges as $\gamma\to 1$ on $\Pi\times\mathcal{P}$. 
Also, $S^\pi_\infty(\kp,\gamma)$ is $L$-Lipschitz for any $\gamma>\delta$: for any $\pi,\kp$ and any $\gamma_1,\gamma_2\in (\delta,1]$.
\begin{align}
     \|S^\pi_\infty(\kp,\gamma_1)-S^\pi_\infty(\kp,\gamma_2)\|\leq L|\gamma_1-\gamma_2|. 
\end{align}
\end{lemma}
\begin{proof}
From Lemma \ref{lemma:Sn}, for any $\epsilon$, there exists $N_\epsilon$, such that for any $n\geq N_\epsilon$, $\pi,\kp$, $\gamma>\delta$, 
\begin{align}
    \|S^\pi_\infty(\kp,\gamma)-S^\pi_n(\kp,\gamma)\|<\epsilon.
\end{align}
Thus for any $\gamma_1,\gamma_2\in(\delta,1]$, 
\begin{align}\label{eq:31}
    &\|S^\pi_\infty(\kp,\gamma_1)-S^\pi_\infty(\kp,\gamma_2)\|\nn\\
    &\leq \|S^\pi_\infty(\kp,\gamma_1)-S^\pi_n(\kp,\gamma_1)\|+\|S^\pi_n(\kp,\gamma_1)-S^\pi_n(\kp,\gamma_2)\|+\|S^\pi_n(\kp,\gamma_2)-S^\pi_\infty(\kp,\gamma_2)\|\nn\\
    &\leq 2\epsilon+ \|S^\pi_n(\kp,\gamma_1)-S^\pi_n(\kp,\gamma_2)\|\nn\\
    &\leq 2\epsilon +L|\gamma_1-\gamma_2|, 
\end{align}
where the last step is from Lemma \ref{lemma:Snlip}. 

Thus, for any $\epsilon$, there exists $\omega=\max\left\{ \delta, 1-\epsilon\right\}$, such that for any $\gamma_1,\gamma_2>\omega$, 
\begin{align}
    \|S^\pi_\infty(\kp,\gamma_1)-S^\pi_\infty(\kp,\gamma_2)\|<(2+L)\epsilon,
\end{align}
and hence by Cauchy's criterion we conclude that $S^\pi_\infty(\kp,\gamma)$ converges uniformly on $\Pi\times\mathcal{P}$.

On the other hand, since \cref{eq:31} holds for any $\epsilon$, it implies that 
\begin{align}
     \|S^\pi_\infty(\kp,\gamma_1)-S^\pi_\infty(\kp,\gamma_2)\|\leq L|\gamma_1-\gamma_2|,
\end{align}
which completes the proof.
\end{proof}


We now prove Theorem \ref{thm:uniform}.
For any $\kp,\pi$, we have that
\begin{align}
    V^\pi_{\kp,\gamma}=\frac{1}{1-\gamma}g^\pi_\kp+h^\pi_\kp+f^\pi_\kp(\gamma).
\end{align}
It then follows that
\begin{align}
    (1-\gamma)V^\pi_{\kp,\gamma}=g^\pi_\kp+(1-\gamma)h^\pi_\kp+(1-\gamma)f^\pi_\kp(\gamma).
\end{align}
Clearly $(1-\gamma)h^\pi_\kp\to 0$ uniformly on $\Pi\times\mathcal{P}$ because $\|h^\pi_\kp\|=\|H^\pi_\kp r_\pi\|\leq h$ is uniformly bounded. Then, 
\begin{align}
    &\|(1-\gamma_1)f^\pi_\kp(\gamma_1)-(1-\gamma_2)f^\pi_\kp(\gamma_2)\|\nn\\
    &\leq \|(1-\gamma_1)f^\pi_\kp(\gamma_1)-(1-\gamma_1)f^\pi_\kp(\gamma_2)\|+\|(1-\gamma_1)f^\pi_\kp(\gamma_2)-(1-\gamma_2)f^\pi_\kp(\gamma_2)\|\nn\\
    &\leq (1-\gamma_1)L|\gamma_1-\gamma_2|+\|f^\pi_\kp(\gamma_2)\| |\gamma_1-\gamma_2|.
\end{align}
For any $\pi,\kp,\gamma>\delta$, 
\begin{align}\label{eq:cf}
    \|f^\pi_\kp(\gamma)\|&=\bigg\| \frac{1}{\gamma}\sum^\infty_{n=1} (-1)^n\left(\frac{1-\gamma}{\gamma} \right)^n (H^\pi_\kp)^{n+1} r_\pi\bigg\|\nn\\
    &\leq \bigg|\frac{1}{\gamma}\sum^\infty_{n=1} \left(\frac{1-\gamma}{\gamma}\right)^nh^{n+1} \bigg|\nn\\
    &\leq \frac{h}{\delta} \frac{1-\gamma}{\gamma}h \frac{1}{1-\frac{1-\gamma}{\gamma}h}\nn\\
    &\leq \frac{h}{\delta}   \frac{k}{1-k}\nn\\
    &\triangleq c_f.
\end{align}
Hence,
$(1-\gamma)f^\pi_\kp(\gamma)\to 0$ uniformly on $\Pi\times\mathcal{P}$ due to the fact  that $\|f^\pi_\kp(\gamma)\|$ is uniformly bounded for any $\pi,\gamma>\delta,\kp$.

Then we have that
$
    \lim_{\gamma\to 1}(1-\gamma)V^\pi_{\kp,\gamma}=g^\pi_\kp
$
uniformly on $\mathcal{P}\times \Pi$. This completes the proof of \Cref{thm:uniform}.

\section{Proof of Theorem \ref{thm:lim of robust}}


We first show a lemma which allows us to interchange the order of $\lim$ and $\max$. 

\begin{lemma}\label{lemma:limmax}
If a function $f(x,y)$ converges uniformly to $F(x)$ on $\mathcal{X}$ as $y\to y_0$, then 
\begin{align}
    \max_x\lim_{y\to y_0} f(x,y) = \lim_{y\to y_0}\max_xf(x,y) .
\end{align}
\end{lemma}
\begin{proof}
For each $f(x,y)$, denote by  $\arg\max_x f(x,y)=x_y$, and hence $f(x_y,y)\geq f(x,y)$ for any $x,y$. Also denote by $\arg\max_x F(x)=x'$.
Now because $f(x,y)$ uniformly converges to $F(x)$, then for any $\epsilon$, there exists $\delta'$, such that $\forall |y-y_0|<\delta'$, 
\begin{align}
    |f(x,y)-F(x)|\leq\epsilon
\end{align}
for any $x$. 
Now consider $|f(x_y,y)-F(x')|$ for $|y-y_0|<\delta'$. If  $f(x_y,y)-F(x')>0$, then 
\begin{align}
   |f(x_y,y)-F(x')|=f(x_y,y)-F(x')=f(x_y,y)-F(x_y)+F(x_y)-F(x')\leq \epsilon;
\end{align}
On the other hand if $f(x_y,y)-F(x')<0$, then
\begin{align}
   |f(x_y,y)-F(x')|=F(x')-f(x_y,y)=F(x')-f(x',y)+f(x',y)-f(x_y,y)\leq \epsilon.
\end{align}
Hence, we showed that for any $\epsilon$, there exists $\delta'$, such that $\forall |y-y_0|<\delta'$,
\begin{align}
    |f(x_y,y)-F(x')|=|\max_x f(x,y) -\max_x F(x)|\leq \epsilon,
\end{align}
and hence 
\begin{align}
    \lim_{y\to y_0} \max_x f(x,y) =\max_x F(x)=\max_x \lim_{y\to y_0}  f(x,y),
\end{align}
and this completes the proof. 
\end{proof}

Then, we show that the robust discounted value function converges uniformly to the robust average-reward as the discounted factor approaches $1$. 
\begin{theorem}[Restatement of Theorem \ref{thm:lim of robust}]
The  robust discounted value function converges uniformly to the robust average-reward  on $\Pi$: 
\begin{align}
    \lim_{\gamma\to 1}(1-\gamma)V^\pi_{\mathcal{P},\gamma}=g^\pi_{\mathcal{P}}.
\end{align}
\end{theorem}
\begin{proof}
Due to Theorem \ref{thm:stationary}, for any stationary policy $\pi$, $g^\pi_\cp(s)=\min_{\kp\in\cp}g^\pi_\kp(s)$ under the stationary model. Hence from the uniform convergence in Theorem \ref{thm:uniform}, we first show the following:
\begin{align}\label{eq:54}
    g^\pi_{\mathcal{P}}&=\min_{\kp\in\mathcal{P}} g^\pi_\kp\nn\\
    &=\min_{\kp\in\mathcal{P}}\lim_{\gamma\to 1}(1-\gamma)V^\pi_{\kp,\gamma}\nn\\
    &\overset{(a)}{=}\lim_{\gamma\to 1}\min_{\kp\in\mathcal{P}}(1-\gamma)V^\pi_{\kp,\gamma}\nn\\
    &=\lim_{\gamma\to 1}(1-\gamma)V^\pi_{\mathcal{P},\gamma},
\end{align}
where $(a)$ is because Lemma \ref{lemma:limmax}. Moreover, note that $\lim_{\gamma\to 1}(1-\gamma)V^\pi_{\kp,\gamma}=g^\pi_\kp$ uniformly on $\Pi\times\cp$, hence the convergence in \eqref{eq:54} is also uniform on $\Pi$.  
Thus, we complete the proof. 
\end{proof}


\section{Proof of Theorem \ref{thm:evaluation}}

\begin{theorem}[Restatement of Theorem \ref{thm:evaluation}]
$V_T$ generated by Algorithm \ref{alg:evaluatoin} converges to the robust average-reward $g^\pi_{\mathcal{P}}$ as $T\to\infty$.
\end{theorem}

\begin{proof}
From discounted robust Bellman equation \cite{nilim2004robustness}, it can be shown that 
\begin{align}
    (1-\gamma_t)V^\pi_{\cp,\gamma_{t}}=(1-\gamma_t)\sum_a \pi(a|s) (r(s,a)+\gamma_t \sigma_{\mathcal{P}^a_s}( V^\pi_{\cp,\gamma_{t}})).
\end{align}
Then we can show that for any $s\in \mcs$,
\begin{align}
    &|V_{t+1}(s)-(1-\gamma_{t+1})V^\pi_{\cp,\gamma_{t+1}}(s)| \nn\\
    &=|V_{t+1}(s)-(1-\gamma_t)V^\pi_{\cp,\gamma_{t}}(s)+(1-\gamma_t)V^\pi_{\cp,\gamma_{t}}(s)-(1-\gamma_{t+1})V^\pi_{\cp,\gamma_{t+1}}(s)|\\
    &\leq |(1-\gamma_t)V^\pi_{\cp,\gamma_{t}}(s)-(1-\gamma_{t+1})V^\pi_{\cp,\gamma_{t+1}}(s)|+|V_{t+1}(s)-(1-\gamma_t)V^\pi_{\cp,\gamma_{t}}(s)|\nn\\
    &=|(1-\gamma_t)V^\pi_{\cp,\gamma_{t}}(s)-(1-\gamma_{t+1})V^\pi_{\cp,\gamma_{t+1}}(s)|\nn\\
    &\quad+\bigg|\sum_a \pi(a|s) \bigg((1-\gamma_t)r(s,a)+\gamma_t \sigma_{\mathcal{P}^a_s}(V_t)-((1-\gamma_{t})r(s,a)+\gamma_{t} \sigma_{\mathcal{P}^a_s}((1-\gamma_t)V^\pi_{\cp,\gamma_{t}})) \bigg)\bigg|\nn\\
    &=|(1-\gamma_t)V^\pi_{\cp,\gamma_{t}}(s)-(1-\gamma_{t+1})V^\pi_{\cp,\gamma_{t+1}}(s)|+\bigg|\sum_a \pi(a|s) \bigg(\gamma_t\sigma_{\mathcal{P}^a_s}(V_t)-\gamma_{t} \sigma_{\mathcal{P}^a_s}((1-\gamma_t)V^\pi_{\cp,\gamma_{t}}) \bigg)\bigg|\nn\\
    &=|(1-\gamma_t)V^\pi_{\cp,\gamma_{t}}(s)-(1-\gamma_{t+1})V^\pi_{\cp,\gamma_{t+1}}(s)|+\gamma_{t}\bigg|\sum_a \pi(a|s) \bigg(\sigma_{\mathcal{P}^a_s}(V_t)- \sigma_{\mathcal{P}^a_s}((1-\gamma_t)V^\pi_{\cp,\gamma_{t}})\bigg)\bigg|.
\end{align}
If we denote by $\Delta_t\triangleq \|V_{t}-(1-\gamma_t)V^\pi_{\cp,\gamma_{t}}\|_\infty$, then 
\begin{align}\label{eq:56}
    \Delta_{t+1}\leq \|(1-\gamma_t)V^\pi_{\cp,\gamma_{t}}-(1-\gamma_{t+1})V^\pi_{\cp,\gamma_{t+1}}\|_\infty +\gamma_{t}\max_s\bigg\{\sum_a \pi(a|s) \bigg|\sigma_{\mathcal{P}^a_s}(V_t)- \sigma_{\mathcal{P}^a_s}((1-\gamma_t)V^\pi_{\cp,\gamma_{t}})\bigg|\bigg\}.
\end{align}

It can be easily verified that $\sigma_{\mathcal{P}^a_s}(V)$ is a $1$-Lipschitz function, thus the second term in \eqref{eq:56} can be further bounded as
\begin{align}
    &\sum_a \pi(a|s) \bigg|\sigma_{\mathcal{P}^a_s}(V_t)- \sigma_{\mathcal{P}^a_s}((1-\gamma_t)V^\pi_{\cp,\gamma_{t}})\bigg|\nn\\
    &\leq \sum_a \pi(a|s) \|V_t-(1-\gamma_t)V^\pi_{\cp,\gamma_{t}} \|_\infty\nn\\
    &=\|V_t-(1-\gamma_t)V^\pi_{\cp,\gamma_{t}} \|_\infty,
\end{align}
and hence 
\begin{align}
    \Delta_{t+1}\leq \|(1-\gamma_t)V^\pi_{\cp,\gamma_{t}}-(1-\gamma_{t+1})V^\pi_{\cp,\gamma_{t+1}}\|_\infty +\gamma_{t}\Delta_t.
\end{align}
Recall that 
\begin{align}
    (1-\gamma_t)V^\pi_{\cp,\gamma_{t}}=(1-\gamma_t)\min_{\kp} V^\pi_{\kp,\gamma_t}.
\end{align}
Let $s_t^*\triangleq\arg\max_s|(1-\gamma_t)V^\pi_{\cp,\gamma_{t}}(s)-(1-\gamma_{t+1})V^\pi_{\cp,\gamma_{t+1}}(s)| $. Then it follows that
\begin{align}
    \|(1-\gamma_t)V^\pi_{\cp,\gamma_{t}}-(1-\gamma_{t+1})V^\pi_{\cp,\gamma_{t+1}}\|_\infty=|(1-\gamma_t)V^\pi_{\cp,\gamma_{t}}(s_t^*)-(1-\gamma_{t+1})V^\pi_{\cp,\gamma_{t+1}}(s_t^*)|.
\end{align}

Note that from \cite{nilim2004robustness,iyengar2005robust}, for any stationary policy $\pi$, there exists a stationary model $\kp$ such that $V^\pi_{\cp,\gamma}(s)=\mE_{\kp,\pi}\bigg[ \sum^{\infty}_{t=0}\gamma^tr_t|S_0=s\bigg]\triangleq V^\pi_{\kp,\gamma}$. Hence in the following, for each $\gamma_t$, we denote the worst-case transition kernel of $V^\pi_{\cp,\gamma_{t}}$ by $\kp_t$. 

If $(1-\gamma_t)V^\pi_{\cp,\gamma_{t}}(s_t^*)\geq (1-\gamma_{t+1})V^\pi_{\cp,\gamma_{t+1}}(s_t^*)$, then
\begin{align}\label{eq:36}
    &|(1-\gamma_t)V^\pi_{\cp,\gamma_{t}}(s_t^*)-(1-\gamma_{t+1})V^\pi_{\cp,\gamma_{t+1}}(s_t^*)|\nn\\&=\min_{\kp} (1-\gamma_{t})V^\pi_{\kp,\gamma_t}(s_t^*)-\min_{\kp} (1-\gamma_{t+1})V^\pi_{\kp,\gamma_{t+1}}(s_t^*)\nn\\
    &=(1-\gamma_t){V}^\pi_{\kp_t,\gamma_t}(s_t^*)- (1-\gamma_{t+1}){V}^\pi_{\kp_{t+1},\gamma_{t+1}}(s_t^*)\nn\\
    &=(1-\gamma_t){V}^\pi_{\kp_t,\gamma_t}(s_t^*)-(1-\gamma_t){V}^\pi_{\kp_{t+1},\gamma_{t}}(s_t^*)+(1-\gamma_t){V}^\pi_{\kp_{t+1},\gamma_{t}}(s_t^*)- (1-\gamma_{t+1}){V}^\pi_{\kp_{t+1},\gamma_{t+1}}(s_t^*)\nn\\
    &\overset{(a)}{\leq} (1-\gamma_t){V}^\pi_{\kp_{t+1},\gamma_t}(s_t^*)-(1-\gamma_{t+1}){V}^\pi_{\kp_{t+1},\gamma_{t+1}}(s_t^*)\nn\\
    &\leq \|(1-\gamma_t){V}^\pi_{\kp_{t+1},\gamma_t}-(1-\gamma_{t+1}){V}^\pi_{\kp_{t+1},\gamma_{t+1}}\|_\infty,
\end{align}
where   $(a)$ is due to $(1-\gamma_t){V}^\pi_{\kp_{t},\gamma_{t}}(s_t^*)=\min_{\kp} (1-\gamma_t)V^\pi_{\kp,\gamma_{t}}(s_t^*)\leq (1-\gamma_t)V^\pi_{\kp_{t+1},\gamma_{t}}(s_t^*)$.

Now, according to Lemma \ref{lemma:dis-ave},
\begin{align}
    (1-\gamma_t){V}^\pi_{\kp_{t+1},\gamma_t}&=g^\pi_{\kp_{t+1}}+(1-\gamma_t)h^\pi_{\kp_{t+1}}+(1-\gamma_t)f^\pi_{\kp_{t+1}}(\gamma_t),\\
    (1-\gamma_{t+1}){V}^\pi_{\kp_{t+1},\gamma_{t+1}}&=g^\pi_{\kp_{t+1}}+(1-\gamma_{t+1})h^\pi_{\kp_{t+1}}+(1-\gamma_{t+1})f^\pi_{\kp_{t+1}}(\gamma_{t+1}).
\end{align}
Hence, for any $\gamma_t>\delta$, \cref{eq:36} can be further bounded as
\begin{align}\label{eq:65}
     &\|(1-\gamma_t){V}^\pi_{\kp_{t+1},\gamma_t}-(1-\gamma_{t+1}){V}^\pi_{\kp_{t+1},\gamma_{t+1}}\|_\infty\nn\\
    &=\|(\gamma_{t+1}-\gamma_t)h^\pi_{\kp_{t+1}}+(1-\gamma_t)f^\pi_{\kp_{t+1}}(\gamma_t)-(1-\gamma_{t+1})f^\pi_{\kp_{t+1}}(\gamma_{t+1})\|_\infty\nn\\
    &\leq (\gamma_{t+1}-\gamma_t)\|h^\pi_{\kp_{t+1}}\|_\infty+ \|f^\pi_{\kp_{t+1}}(\gamma_t)-f^\pi_{\kp_{t+1}}(\gamma_{t+1})\|_\infty+\| \gamma_{t+1}f^\pi_{\kp_{t+1}}(\gamma_{t+1})-\gamma_tf^\pi_{\kp_{t+1}}(\gamma_{t})\|_\infty\nn\\
    &\overset{(a)}{\leq} h (\gamma_{t+1}-\gamma_t)+L(\gamma_{t+1}-\gamma_t)+\| \gamma_{t+1}f^\pi_{\kp_{t+1}}(\gamma_{t+1})-\gamma_tf^\pi_{\kp_{t+1}}(\gamma_{t})\|_\infty\nn\\
    &\leq h(\gamma_{t+1}-\gamma_t)+L(\gamma_{t+1}-\gamma_t)+\| \gamma_{t+1}f^\pi_{\kp_{t+1}}(\gamma_{t+1})-\gamma_{t+1}f^\pi_{\kp_{t+1}}(\gamma_{t})\|_\infty+\|\gamma_{t+1}f^\pi_{\kp_{t+1}}(\gamma_{t})-\gamma_{t}f^\pi_{\kp_{t+1}}(\gamma_{t}) \|_\infty\nn\\
    &\leq h(\gamma_{t+1}-\gamma_t)+L(\gamma_{t+1}-\gamma_t)+\gamma_{t+1}\|f^\pi_{\kp_{t+1}}(\gamma_{t+1})-f^\pi_{\kp_{t+1}}(\gamma_{t})\|_\infty+\|f^\pi_{\kp_{t+1}}(\gamma_{t}) \|_\infty (\gamma_{t+1}-\gamma_t)\nn\\
    &\overset{(b)}{\leq} (h+L+\gamma_{t+1}L+\sup_{\pi,\kp,\gamma} \|f^\pi_{\kp}(\gamma) \|_\infty) (\gamma_{t+1}-\gamma_t)\nn\\
    &\leq K  (\gamma_{t+1}-\gamma_t),
\end{align}
where $(a)$ is from Lemma \ref{lemma:lip_f} for any $\gamma_t>\delta$, $c_f$ is defined in \eqref{eq:cf} and $K\triangleq h+2L+c_f$ is a uniform constant; And $(b)$ is from Lemma \ref{lemma:lip_f}. 

Similarly, the inequality also holds for the case when $(1-\gamma_t)V^\pi_{\cp,\gamma_{t}}(s_t^*)\leq (1-\gamma_{t+1})V^\pi_{\cp,\gamma_{t+1}}(s_t^*)$. Thus we have that for any $t$ such that $\gamma_t>\delta$, 
\begin{align}
    \Delta_{t+1}\leq K(\gamma_{t+1}-\gamma_t)+\gamma_{t}\Delta_t,
\end{align}
where $K$ is a uniform constant. 

Following Lemma 8 from \cite{tewari2007bounded}, we have that $\Delta_t\to 0$. Note that  
\begin{align}
    \| V_t-g^\pi_{\mathcal{P}}\|_\infty\leq \|V_t- (1-\gamma_t)V^\pi_{\cp,\gamma_{t}}\|_\infty +\|(1-\gamma_t)V^\pi_{\cp,\gamma_{t}}-g_{\mathcal{P}}^\pi \|_\infty=\Delta_t+\|(1-\gamma_t)V^\pi_{\cp,\gamma_{t}}-g_{\mathcal{P}}^\pi \|_\infty.
\end{align}
Together with Theorem \ref{thm:lim of robust}, we further have that 
\begin{align}
    \lim_{t\to\infty}\|V_t-g_\cp^\pi\|_\infty=0,
\end{align}
which completes the proof.

\end{proof}

\section{Proof of Theorem \ref{thm:val_converges}}

Note that the  optimal robust average-reward is defined as 
\begin{align}
    g^*_{\cp}(s)\triangleq \max_{\pi} g^\pi_{\cp}(s).
\end{align}
We further define
\begin{align}
    V^*_{\cp,\gamma}(s)\triangleq \max_{\pi} V^\pi_{\cp,\gamma}(s).
\end{align}

\begin{theorem}[Restatement of Theorem \ref{thm:val_converges}]
$V_T$ generated by Algorithm \ref{alg:valueiteration} converges to the optimal robust average-reward $g^*_\cp$ as $T\to\infty$. 
\end{theorem}
\begin{proof}

Firstly, from the uniform convergence in Theorem \ref{thm:lim of robust}, it can be shown that 

\begin{align}\label{eq:limv_t^*}
    \lim_{t\to\infty}(1-\gamma_t)V^*_{\cp,\gamma_t}=g^*_\cp.
\end{align}
We then show that for any $s\in\mcs$,
\begin{align}\label{eq:51}
    &|V_{t+1}(s)-(1-\gamma_{t+1})V^*_{\cp,\gamma_{t+1}}(s)| \nn\\
    &\leq|V_{t+1}(s)-(1-\gamma_t)V^*_{\cp,\gamma_t}(s)|+|(1-\gamma_t)V^*_{\cp,\gamma_t}(s)-(1-\gamma_{t+1})V^*_{\cp,\gamma_{t+1}}(s)|\nn\\
    &\overset{(a)}{=}|(1-\gamma_t)V^*_{\cp,\gamma_t}(s)-(1-\gamma_{t+1})V^*_{\cp,\gamma_{t+1}}(s)|\nn\\
    &\quad+\bigg|\max_a \bigg((1-\gamma_t)r(s,a)+\gamma_t \sigma_{\mathcal{P}^a_s}(V_t)\bigg)-\max_a\bigg(((1-\gamma_{t})r(s,a)+\gamma_{t} \sigma_{\mathcal{P}^a_s}((1-\gamma_t)V^*_{\cp,\gamma_t})) \bigg)\bigg|\nn\\
    &\leq|(1-\gamma_t)V^*_{\cp,\gamma_t}(s)-(1-\gamma_{t+1})V^*_{\cp,\gamma_{t+1}}(s)|\nn\\
    &\quad+\max_a \bigg|(1-\gamma_t)r(s,a)+\gamma_t \sigma_{\mathcal{P}^a_s}(V_t)-((1-\gamma_{t})r(s,a)+\gamma_{t} \sigma_{\mathcal{P}^a_s}((1-\gamma_t)V^*_{\cp,\gamma_t})) \bigg|,
\end{align}
where $(a)$ is because the optimal robust Bellman equation, and  the last inequality is from the fact that  $|\max_x f(x)-\max_x g(x)|\leq \max_x|f(x)-g(x)|$. 

Hence \cref{eq:51} can be further bounded as
\begin{align}
    &|V_{t+1}(s)-(1-\gamma_{t+1})V^*_{\cp,\gamma_{t+1}}(s)| \nn\\
    &\leq |(1-\gamma_t)V^*_{\cp,\gamma_t}(s)-(1-\gamma_{t+1})V^*_{\cp,\gamma_{t+1}}(s)|+\gamma_t\max_a \bigg| \sigma_{\mathcal{P}^a_s}(V_t)- \sigma_{\mathcal{P}^a_s}((1-\gamma_t)V^*_{\cp,\gamma_t}) \bigg|.
\end{align}

If we denote by $\Delta_t\triangleq \|V_{t}-(1-\gamma_t)V^*_{\cp,\gamma_t}\|_\infty$, then 
\begin{align}
    \Delta_{t+1}\leq \|(1-\gamma_t)V^*_{\cp,\gamma_t}-(1-\gamma_{t+1})V^*_{\cp,\gamma_{t+1}}\|_\infty +\gamma_{t}\max_{s.a}  \bigg|\sigma_{\mathcal{P}^a_s}(V_t)- \sigma_{\mathcal{P}^a_s}((1-\gamma_t)V^*_{\cp,\gamma_t})\bigg|.
\end{align}
Since the support function $\sigma_{\cp^a_s}(V)$ is $1$-Lipschitz, then it can be shown that for any $s,a$, 
\begin{align}
    \bigg|\sigma_{\mathcal{P}^a_s}(V_t)- \sigma_{\mathcal{P}^a_s}((1-\gamma_t)V^*_{\cp,\gamma_t})\bigg|\leq \|V_t-(1-\gamma_t)V^*_{\cp,\gamma_t} \|_\infty.
\end{align}
Hence 
\begin{align}
    \Delta_{t+1}\leq \|(1-\gamma_t)V^*_{\cp,\gamma_t}-(1-\gamma_{t+1})V^*_{\cp,\gamma_{t+1}}\|_\infty +\gamma_t \Delta_t.
\end{align}
Similar to \eqref{eq:65} in Theorem \ref{thm:evaluation}, we can show that 
\begin{align}
    \|(1-\gamma_t)V^*_{\cp,\gamma_t}-(1-\gamma_{t+1})V^*_{\cp,\gamma_{t+1}}\|_\infty\leq K|\gamma_t-\gamma_{t+1}|,
\end{align}
and similar to Lemma 8 from \cite{tewari2007bounded}, 
\begin{align}
    \lim_{t\to\infty} \Delta_t=0. 
\end{align}
Moreover, note that 
\begin{align}
    \|V_t-g^*_\cp \|_\infty\leq \|V_t-(1-\gamma_t)V^*_{\cp,\gamma_t} \|_\infty+ \|(1-\gamma_t)V^*_{\cp,\gamma_t}-g^*_\cp \|_\infty=\Delta_t+\|(1-\gamma_t)V^*_{\cp,\gamma_t}-g^*_\cp \|_\infty,
\end{align}
which together with \cref{eq:limv_t^*} implies that 
\begin{align}
    \|V_t-g^*_\cp \|_\infty\to 0,
\end{align}
and hence it completes the proof.

\end{proof}
\begin{lemma}\label{lemma:deter opt policy}
There exists a deterministic optimal policy, i.e., $\exists \pi^*\in\Pi_D$, s.t. $g^{\pi^*}_\cp=g^*_\cp=\max_{\pi\in\Pi}g^\pi_\cp$. 
\end{lemma}
\section{Proof of Lemma \ref{lemma:deter opt policy}}
\begin{lemma}(Restatement of \Cref{lemma:deter opt policy}).
There exists a deterministic optimal policy, i.e., $\exists \pi^*\in\Pi_D$, s.t. $g^{\pi^*}_\cp=g^*_\cp=\max_{\pi\in\Pi}g^\pi_\cp$. 
\end{lemma}
\begin{proof}
Assume that there is no deterministic optimal robust policy, i.e., there exists a strictly random policy $\pi_r\in\Pi$, such that for any deterministic policy $\pi\in\Pi_D$, 
\begin{align}
    g^{\pi_r}_\cp > g^\pi_\cp. 
\end{align}
According to \cref{thm:lim of robust}, we have that
\begin{align}
    \lim_{\gamma\to 1} (1-\gamma)V^{\pi_r}_{\cp,\gamma} &= g^{\pi_r}_\cp,\\
    \lim_{\gamma\to 1} (1-\gamma)V^{\pi}_{\cp,\gamma} &= g^{\pi}_\cp, \forall \pi\in\Pi_D. 
\end{align}
Since there are only finite number of deterministic policies, there exists $\delta<1$, such that for any $\gamma>\delta$, 
\begin{align}
    V^{\pi_r}_{\cp,\gamma} > V^{\pi}_{\cp,\gamma}, \forall \pi\in\Pi_D. 
\end{align}
This implies that for $\gamma>\delta$, the random policy $\pi_r$ is better than all the deterministic policies, i.e., 
\begin{align}\label{eq:91}
    V^{\pi_r}_{\cp,\gamma} > \max_{\pi\in\Pi_D}V^{\pi}_{\cp,\gamma}. 
\end{align}
However, Theorem 3.1 of \cite{iyengar2005robust} implies that there exists deterministic optimal robust policy, i.e., 
\begin{align}
    \max_{\pi\in\Pi_D}V^{\pi}_{\cp,\gamma}=\max_{\pi\in\Pi}V^{\pi}_{\cp,\gamma}\geq V^{\pi_r}_{\cp,\gamma},
\end{align}
which contradicts to \eqref{eq:91}. Hence it implies that there exists a deterministic optimal robust policy, and completes the proof.
\end{proof}

\section{Proof of Theorem \ref{thm:blackwell}}
\begin{theorem}[Restatement of Theorem \ref{thm:blackwell}]
There  exists $0<\delta<1$, such that for any $\gamma>\delta$, a deterministic optimal robust policy for robust discounted value function $V^*_{\cp,\gamma}$ is also an optimal policy for robust average-reward, i.e.,
\begin{align}
    V^{\pi^*}_{\cp,\gamma}=V^*_{\cp,\gamma}.
\end{align}
Moreover, 
when $\arg\max_{\pi\in\Pi^D} g^\pi_\cp$ is a singleton, there exists a unique Blackwell optimal policy. 
\end{theorem}
\begin{proof}
According to \Cref{lemma:deter opt policy}, there exists $\pi^*\in\Pi^D$ such that 
\begin{align}
    g^*_\cp=g^{\pi^*}_\cp.
\end{align}
Assume the robust average-reward of all deterministic policies are sorted in a descending order:
\begin{align}
    g^*_\cp=g^{\pi^*_1}_\cp=g^{\pi^*_2}_\cp=...=g^{\pi^*_m}_\cp > g^{\pi_1}_\cp\geq...\geq g^{\pi_n}_\cp
\end{align}
for all $\pi_i^*,\pi_i\in\Pi^D$, and we define $\Pi^*=\{\pi^*_i:i=1,...,m \}$.
Denote by $d=g^{\pi_i^*}_\cp - g^{\pi_1}_\cp$.

From Theorem \ref{thm:lim of robust}, we know that for any $\pi\in\Pi^D$, 
\begin{align}
    \lim_{\gamma\to 1} (1-\gamma)V^\pi_{\cp,\gamma}=g^\pi_\cp. 
\end{align}
Because the set $\Pi^D$ is finite, for any $\epsilon<\frac{d}{2}$, there exists $\delta'<1$, such that for any $\gamma>\delta'$, $\pi^*_i$ and $\pi_j$, \begin{align}
   | (1-\gamma)V^{\pi_i^*}_{\cp,\gamma}-g^*_\cp|<\epsilon,\\
   | (1-\gamma)V^{\pi_j}_{\cp,\gamma}-g^{\pi_j}_\cp|<\epsilon.
\end{align}
It hence implies that 
\begin{align}
    (1-\gamma)V^{\pi_i^*}_{\cp,\gamma}\geq(d-2\epsilon)+(1-\gamma)V^{{\pi_j}}_{\cp,\gamma}>(1-\gamma)V^{{\pi_j}}_{\cp,\gamma},
\end{align}
and 
\begin{align}\label{eq:vorder}
   V^{\pi_i^*}_{\cp,\gamma}>V^{{\pi_j}}_{\cp,\gamma}.
\end{align}

Note that from Theorem 3.1 in \cite{iyengar2005robust}, i.e., $\max_{\pi\in\Pi^D}V^\pi_{\cp,\gamma}=V^*_{\cp,\gamma}$, we have that for any $\gamma$, there exists a deterministic policy $\pi\in\Pi^D$, such that $V^*_{\cp,\gamma}=V^\pi_{\cp,\gamma}$. Together with \eqref{eq:vorder}, it implies that all the possible optimal robust polices of $V^\pi_{\cp,\gamma}$ belong to $\left\{ \pi^*_1,...\pi^*_m\right\}$, i.e., the set $\Pi^*$.  Hence, there exists $\pi^*_j\in\Pi^*$, such that 
\begin{align}
    V^{\pi_j^*}_{\cp,\gamma}=\max_{\pi\in\Pi^D}V^\pi_{\cp,\gamma}=V^*_{\cp,\gamma}.
\end{align}

For the second part, when the optimal robust policy of robust average-reward is unique, i.e., $\Pi^*=\left\{ \pi^*\right\}$. Then from the results above, there exists $\delta'$, such that for any $\gamma>\delta'$, $V^{\pi^*}_{\cp,\gamma}>V^{\pi}_{\cp,\gamma}$ for any $\pi^*\neq \pi\in\Pi^D$, and hence $\pi^*$ is the optimal policy for discounted robust MDPs, which is the unique Blackwell optimal policy.

\end{proof}

\section{Proof of Results for Direct Approach}
Recall that
\begin{align}
     V^\pi_\cp(s)\triangleq\min_{\kappa\in\bigotimes_{t\geq 0}\cp}\mathbb{E}_{\kappa,\pi}\bigg[\sum^\infty_{t=0}(r_t-g^\pi_\cp)\big |S_0=s\bigg],
\end{align}
where 
\begin{align}
    g^\pi_\cp=\min_{\kappa\in\bigotimes_{t\geq 0} \mathcal{P}} \lim_{n\to\infty}\mathbb{E}_{\kappa,\pi}\left[\frac{1}{n}\sum_{t=0}^{n-1}r_t|S_0=s\right].
\end{align}

We first show that the robust relative function is always finite. 
\begin{lemma}
For any $\pi$, $V^\pi_\cp$ is finite. 
\end{lemma}
\begin{proof}
According to Theorem \ref{thm:stationary}, $V^\pi_\cp=\min_{\kp\in\cp} V^\pi_\kp=\min_{\kp\in\cp}\mathbb{E}_{\kp,\pi}\bigg[\sum^\infty_{t=0}(r_t-g^\pi_\cp) \bigg]$. 
Note that $V^\pi_\cp$ can be rewritten as
\begin{align}
    V^\pi_\cp&=\min_{\kp\in\cp}\mathbb{E}_{\kp,\pi}\bigg[\sum^\infty_{t=0}(r_t-g^\pi_\cp)\bigg]\nn\\
    &=\min_{\kp\in\cp}\mathbb{E}_{\kp,\pi}\bigg[\lim_{n\to\infty}\sum^n_{t=0}(r_t-g^\pi_\cp)\bigg]\nn\\
    &=\min_{\kp\in\cp}\mathbb{E}_{\kp,\pi}\bigg[\lim_{n\to\infty}\sum^n_{t=0}(r_t-g^\pi_\kp+g^\pi_\kp-g^\pi_\cp)\bigg]\nn\\
    &=\min_{\kp\in\cp}\mathbb{E}_{\kp,\pi}\bigg[\lim_{n\to\infty}(R_n-ng^\pi_\kp+ng^\pi_\kp-ng^\pi_\cp)\bigg],
\end{align}
where $R_n=\sum^{n}_{t=0}r_t$. 
Note that for any $\kp\in\cp$ and $n$, $ng^\pi_\kp\geq ng^\pi_\cp$, hence 
\begin{align}
    \lim_{n\to\infty}(R_n-ng^\pi_\kp+ng^\pi_\kp-ng^\pi_\cp)\geq \lim_{n\to\infty}(R_n-ng^\pi_\kp),
\end{align}
and thus the lower bound of $V^\pi_\cp$ can be derived as follows,
\begin{align}
    V^\pi_\cp&\geq \min_{\kp\in\cp}\mathbb{E}_{\kp,\pi}\bigg[\sum^\infty_{t=0}(r_t-g^\pi_\kp)\bigg]\nn\\
    &=\min_{\kp\in\cp}V^\pi_\kp\nn\\
    &=\min_{\kp\in\cp}H^\pi_\kp r_\pi.
\end{align}
which is finite due to the fact that $H^\pi_\kp$ is continuous on the compact set $\mathcal{P}$. 

From Theorem \ref{thm:stationary}, we denote the stationary worst-case transition kernel of $g^\pi_\cp$ by $\kp_g$. Then the upper bound of $V^\pi_\cp$ can be bounded by noting that 
\begin{align}
    V^\pi_\cp&= \min_{\kp\in\cp}\mathbb{E}_{\kp,\pi}\bigg[\sum^\infty_{t=0}(r_t-g^\pi_{\kp_g})\bigg]\nn\\
    &\leq \mE_{\kp_g,\pi} \bigg[\sum^\infty_{t=0}(r_t-g^\pi_{\kp_g})\bigg]\nn\\
    &=V^\pi_{\kp_g},
\end{align}
which is also finite and $\kp_g$ denotes the worst-case transition kernel of $g^\pi_\cp$. 
Hence we show that $V^\pi_\cp$ is finite for any $\pi$ and hence complete the proof. 
\end{proof}

After showing that   the robust relative value function is well-defined, we show the following robust Bellman equation for average-reward robust MDPs. 
\begin{theorem}[Restatement of Theorem \ref{thm:bellman}]
For any $s$ and $\pi$, $(V^\pi_\cp,g^\pi_\cp)$ is a solution to the following robust Bellman equation:
\begin{align}
    V(s)+g=\sum_a \pi(a|s)\left(r(s,a)+ \sigma_{\cp^a_s}(V)\right). 
\end{align}
\end{theorem}
\begin{proof}
From the definition, 
\begin{align}
     V^\pi_\cp(s)=\min_{\kappa\in\bigotimes_{t\geq 0}\cp}\mathbb{E}_{\kappa,\pi}\bigg[\sum^\infty_{t=0}(r_t-g^\pi_\cp)\big| S_0=s\bigg],
\end{align}
hence 
\begin{align}
    V^\pi_\cp(s)&=\min_{\kappa\in\bigotimes_{t\geq 0}\cp}\mathbb{E}_{\kappa,\pi}\bigg[\sum^\infty_{t=0}(r_t-g^\pi_\cp)\big| S_0=s\bigg]\nn\\
    &=\min_{\kappa\in\bigotimes_{t\geq 0}\cp}\mathbb{E}_{\kappa,\pi}\bigg[(r_0-g^\pi_\cp)+\sum^\infty_{t=1}(r_t-g^\pi_\cp)\big| S_0=s\bigg]\nn\\
    &=\min_{\kappa\in\bigotimes_{t\geq 0}\cp} \left\{\sum_a \pi(a|s)r(s,a)-g^\pi_\cp +\mathbb{E}_{\kappa,\pi}\bigg[\sum^\infty_{t=1}(r_t-g^\pi_\cp)\big| S_0=s\bigg] \right\}\nn\\
    &=\sum_a \pi(a|s)\left(r(s,a)-g^\pi_\cp \right)+ \min_{\kappa\in\bigotimes_{t\geq 0}\cp}\left\{  \sum_{a,s'} \pi(a|s) \kp^a_{s,s'} \mathbb{E}_{\kappa,\pi}\bigg[\sum^\infty_{t=1}(r_t-g^\pi_\cp)|S_1=s'\bigg] \right\}\nn\\
    &=\sum_a \pi(a|s)\left(r(s,a)-g^\pi_\cp \right)+\min_{\kp_0\in\cp} \min_{\kappa=(\kp_1,...)\in\bigotimes_{t\geq 1}\cp}\left\{  \sum_{a,s'} \pi(a|s) (\kp_0)^a_{s,s'} \mathbb{E}_{\kappa,\pi}\bigg[\sum^\infty_{t=1}(r_t-g^\pi_\cp)|S_1=s'\bigg] \right\}\nn\\
    &=\sum_a \pi(a|s)\left(r(s,a)-g^\pi_\cp \right)+\min_{\kp_0\in\cp} \left\{ \sum_{a,s'} \pi(a|s) (\kp_0)^a_{s,s'} \min_{\kappa=(\kp_1,...)\in\bigotimes_{t\geq 1}\cp}\left\{  \mathbb{E}_{\kappa,\pi}\bigg[\sum^\infty_{t=1}(r_t-g^\pi_\cp)|S_1=s'\bigg] \right\}\right\}\nn\\
    &{=}\sum_a \pi(a|s)\left(r(s,a)-g^\pi_\cp \right)+  \sum_a \pi(a|s) \sum_{s'}\min_{p^a_{s,s'}\in\cp^a_s}p^a_{s,s'} V^\pi_{\cp}(s') \nn\\
    &=\sum_a \pi(a|s)\left(r(s,a)-g^\pi_\cp \right)+  \sum_a \pi(a|s) \sigma_{\cp^a_s}\left(V^\pi_\cp\right)\nn\\
    &=\sum_a \pi(a|s)\left(r(s,a)-g^\pi_\cp + \sigma_{\cp^a_s}(V^\pi_\cp)\right).
\end{align} This hence completes the proof.
\end{proof}

\begin{theorem}\label{thm:optimal}[Restatement of Theorem \ref{thm:opt_eq}, Part 1]
For any $(g,V)$ that is a solution to  
$
   \max_{a}\left\{r(s,a)-g+\sigma_{\cp^a_s}(V)-V(s) \right\}=0, \forall s,
$  then
$
    g=g^*_\cp.
$
\end{theorem}
\begin{proof}
In this proof, for two vectors $v,w\in\mathbb{R}^n$,  $v\geq w$   denotes that $v(s)\geq w(s)$ entry-wise.

Let $B(g,V)(s)\triangleq \max_{a}\left\{r(s,a)-g+\sigma_{\cp^a_s}(V)-V(s) \right\}$. Since $(g,V)$ is a solution to \eqref{eq:opt eq}, hence for any $a\in\mca$  and any $s\in\mcs$,
\begin{align}
    r(s,a)-g+\sigma_{\cp^a_s}(V)-V(s)\leq 0, 
\end{align}
from which it follows that for any policy $\pi$,
\begin{align}\label{eq:105}
    g(s)\geq r_\pi(s)+\sum_a \pi(a|s) \sigma_{\cp^a_s}(V)-V(s)\triangleq r_\pi(s)+\sum_a \pi(a|s)(p^a_s)^\top V-V(s),
\end{align}
where $r_\pi(s)\triangleq \sum_a \pi(a|s)r(s,a)$,    $p^a_s\triangleq \arg\min_{p\in\cp^a_s}p^\top V$, and $\kp_V=\{ p^a_s: s\in\mcs,a\in\mca\}$. We also denotes the state transition matrix induced by $\pi$ and $\kp_V$ by $\kp^\pi_V$.

Using these notations, and rewrite \cref{eq:105}, we have that
\begin{align}\label{eq:106}
    g\mathbbm 1\geq r_\pi+(\kp_V^\pi-I)V.
\end{align}

Since the inequality in \cref{eq:106} holds entry-wise, all entries of $\kp_V^\pi$ are positive, then by multiplying   both sides of \cref{eq:106} by $\kp^\pi_V$, we have that
\begin{align}\label{eq:107}
    g\mathbbm 1=g\kp^\pi_V \mathbbm{1} \geq \kp^\pi_V r_\pi+\kp^\pi_V (\kp^\pi_V-I)V.
\end{align}
Multiplying the both sides of \cref{eq:107} by $\kp^\pi_V$, and repeatedly doing that,  we have that 
\begin{align}
g\mathbbm 1&\geq (\kp^\pi_V)^2r_\pi+ (\kp^\pi_V)^2(\kp^\pi_V-I)V,\\
&\quad\vdots\hspace{2cm} \vdots\\
    g\mathbbm 1&\geq (\kp^\pi_V)^{n-1}r_\pi+ (\kp^\pi_V)^{n-1}(\kp^\pi_V-I)V. \label{eq:110}
\end{align}
Summing up these inequalities from \cref{eq:106} to \cref{eq:110}, we have that 
\begin{align}
    ng \mathbbm 1&\geq (I+\kp^\pi_V+...+(\kp^\pi_V)^{n-1})r_\pi+ (I+\kp^\pi_V+...+(\kp^\pi_V)^{n-1})(\kp^\pi_V-I)V,
\end{align}
and from which, it follows that
\begin{align}
    g\mathbbm 1&\geq \frac{1}{n}(I+\kp^\pi_V+...+(\kp^\pi_V)^{n-1})r_\pi+\frac{1}{n} (I+\kp^\pi_V+...+(\kp^\pi_V)^{n-1})(\kp^\pi_V-I)V \nn\\
    &=\frac{1}{n}(I+\kp^\pi_V+...+(\kp^\pi_V)^{n-1})r_\pi+\frac{1}{n}((\kp^\pi_V)^n-I)V.
\end{align}
It can be easily verified that $\lim_{n\to\infty}\frac{1}{n}((\kp^\pi_V)^n-I)V=0$, and hence it implies that 
\begin{align}\label{eq:102}
    g\mathbbm 1&\geq \lim_{n\to\infty}\frac{1}{n}(I+\kp^\pi_V+...+(\kp^\pi_V)^{n-1})r_\pi\nn\\
    &=\lim_{n\to\infty} \frac{1}{n}\mE_{\kp^\pi_V,\pi}\bigg[\sum^n_{t=0} r_t  \bigg]\nn\\
    &= g^\pi_{\kp^\pi_V}\mathbbm 1\nn\\
    &\geq g^\pi_\cp\mathbbm 1. 
\end{align}
Since \cref{eq:102} holds for any policy $\pi$, it follows that
$
    g\geq g^*_\cp. 
$
On the other hand, since $B(g,V)=0$, there exists a policy $\tau$ such that
\begin{align}
    g\mathbbm 1= r_\tau +(\kp^\tau_V-I)V, 
\end{align}
where $r_\tau, \kp^\tau_V$ are similarly defined as for $\pi$.
From Theorem \ref{thm:stationary}, there exists a stationary transition kernel $\kp_{\text{ave}}^\tau$ such that $g^\tau_\cp=g^\tau_{\kp_{\text{ave}}^\tau}$. We denote the state transition matrix induced by $\tau$ and $\kp_{\text{ave}}^\tau$ by $\kp^\tau$. Then because $\kp^\tau_V$ is the worst-case transition of $V$, it follows that
\begin{align}
    \kp^\tau_V V\leq \kp^\tau V.
\end{align}
Thus
\begin{align}
    g\mathbbm 1\leq r_\tau +(\kp^\tau-I)V.
\end{align}
Similarly, we have that 
\begin{align}
    g\mathbbm 1\leq (\kp^\tau)^{j-1} r_\tau + (\kp^\tau)^{j-1}(\kp^\tau-I)V,
\end{align}
for $j=2,...,n$. 
Summing these inequalities together we have that 
\begin{align}
    ng\mathbbm 1&\leq (I+\kp^\tau+...+(\kp^\tau)^{n-1}) r_\tau + (I+\kp^\tau+...+(\kp^\tau)^{n-1})(\kp^\tau)^{n-1}(\kp^\tau-I)V\nn\\
    &=(I+\kp^\tau+...+(\kp^\tau)^{n-1}) r_\tau + ((\kp^\tau)^n-I) V.
\end{align}
Hence
\begin{align}
    g\mathbbm 1\leq \lim_{n\to\infty} \frac{1}{n}\mE_{\kp_{\text{ave}}^\tau,\tau}\bigg[\sum^n_{t=0} r_t  \bigg]= g^\tau_{\kp_{\text{ave}}^\tau}\mathbbm 1= g^\tau_\cp\mathbbm 1\leq g^*_\cp\mathbbm 1. 
\end{align}
Thus 
$
    g=g^*_\cp,
$ and this concludes the proof.
\end{proof}
 
\begin{theorem}[Restatement of Theorem \ref{thm:opt_eq}, Part 2]
For any $(g,V)$ that is a solution to  
\begin{align} 
   \max_{a}\left\{r(s,a)-g+\sigma_{\cp^a_s}(V)-V(s) \right\}=0, \forall s,
\end{align}
if we set
\begin{align} 
\pi^*(s)=\arg\max_a \left\{ r(s,a)+\sigma_{\cp^a_s}(V)\right\}
\end{align}
for any $s\in\mcs$, then $\pi^*$ is an optimal robust policy. 
\end{theorem}
\begin{proof}
Note that for any stationary policy $\pi$, we denote by $\sigma_{\cp^{\pi}}(V)\triangleq (\sum_a \pi(a|s_1)\sigma_{\cp^{a}_{s_1}}(V),...,\sum_a \pi(a|s_{|\mcs|})\sigma_{\cp^a_{s_{|\mcs|}}}(V))$ being a vector in $\mathbb{R}^{|\mcs|}$. Then \cref{eq:112} is equivalent to  
\begin{align}
    r_{\pi^*}+\sigma_{\cp^{\pi^*}}(V)=\max_{\pi}\left\{r_{\pi}+\sigma_{\cp^{\pi}}(V) \right\}. 
\end{align}

Hence, 
\begin{align}
    r_{\pi^*}-g+\sigma_{\cp^{\pi^*}}(V)-V=\max_{\pi}\left\{r_{\pi}-g+\sigma_{\cp^{\pi}}(V) -V\right\}. 
\end{align}
Since $(g,V)$ is a solution to \eqref{eq:opt eq}, it follows that 
\begin{align}\label{eq:115}
    r_{\pi^*}-g+\sigma_{\cp^{\pi^*}}(V)-V=0. 
\end{align}
According to the robust Bellman equation \cref{eq:bellman}, $(g^{\pi^*}_\cp, V^{\pi^*}_\cp)$ is a solution to \cref{eq:115}. Thus from Theorem \ref{thm:optimal}, $g^{\pi^*}_\cp=g^*_\cp$, and hence $\pi^*$ is an optimal robust policy. 
\end{proof}

\begin{theorem}[Restatement of Theorem \ref{thm:conv of vi}]
$(w_T,V_t)$ in Algorithm \ref{alg:relative} converges to a solution of \cref{eq:opt eq}.  
\end{theorem}

\begin{proof}
We first denote the update operator as \begin{align}
    Lv(s)\triangleq \max_a (r(s,a)+\sigma_{\cp^a_s}(v)). 
\end{align}
Now, consider $sp(Lv-Lu)$. Denote by $\acute{s}\triangleq \arg\max_s (Lv(s)-Lu(s))$ and ${\sm}\triangleq \arg\min_s (Lv(s)-Lu(s))$. Also denote by $a_v\triangleq\arg\max_a (r(\su,a)+\sigma_{\cp^a_{\su}}(v))$ and $a_u\triangleq \arg\max_a (r(\su,a)+\sigma_{\cp^a_{\su}}(u))$ 
Then 
\begin{align}\label{eq:76}
    Lv(\su)-Lu(\su)&=\max_a (r(\su,a)+\sigma_{\cp^a_{\su}}(v))-\max_a (r(\su,a)+\sigma_{\cp^a_{\su}}(u))\nn\\
    &\triangleq r(\su,a_v)+\sigma_{\cp^{a_v}_{\su}}(v)-  (r(\su,a_u)+\sigma_{\cp^{a_u}_{\su}}(u))\nn\\
    &\leq r(\su,a_v)+\sigma_{\cp^{a_v}_{\su}}(v)-  (r(\su,a_v)+\sigma_{\cp^{a_v}_{\su}}(u))\nn\\
    &=\sigma_{\cp^{a_v}_{\su}}(v)-\sigma_{\cp^{a_v}_{\su}}(u)\nn\\
    &\triangleq (p^{a_v,v}_{\su})^\top v -(p^{a_v,u}_{\su})^\top u,
\end{align}
where $p^{a_v,v}_{\su}=\arg\min_{p\in\cp^{a_v}_{\su}} p^\top v$ and $p^{a_v,u}_{\su}=\arg\min_{p\in\cp^{a_v}_{\su}} p^\top u$. 
Thus \cref{eq:76} can be further bounded as 
\begin{align}
    &Lv(\su)-Lu(\su)\nn\\
    &\leq (p^{a_v,v}_{\su})^\top v -(p^{a_v,u}_{\su})^\top u\nn\\
    & \leq (p^{a_v,u}_{\su})^\top (v-u). 
\end{align}
Similarly, 
\begin{align}
    Lv(\sm)-Lu(\sm)\geq (p^{a_u,v}_{\sm})^\top (v-u). 
\end{align}
Thus 
\begin{align}
    sp(Lv-Lu)&\leq (p^{a_v,u}_{\su})^\top (v-u)-(p^{a_u,v}_{\sm})^\top (v-u).
\end{align}
Now denote by $v-u\triangleq (x_1,x_2,...,x_n)$, $p^{a_v,u}_{\su}=(p_1,...,p_n)$
and $p^{a_u,v}_{\sm}=(q_1,...,q_n)$. Further denote by $b_i\triangleq \min \{p_i, q_i \}$ Then 
\begin{align}
    &\sum^n_{i=1} p_ix_i- \sum^n_{i=1} q_ix_i\nn\\
    &=\sum^n_{i=1} (p_i-b_i)x_i -\sum^n_{i=1} (q_i-b_i)x_i\nn\\
    &\leq \sum^n_{i=1} (p_i-b_i)\max \{x_i\} -\sum^n_{i=1} (q_i-b_i)\min\{x_i\}\nn\\
    &=\sum^n_{i=1} (p_i-b_i)sp(x)+ \bigg(\sum^n_{i=1} (p_i-b_i)-\sum^n_{i=1} (q_i-b_i)\bigg)\min\{x_i\}\nn\\
    &=\bigg(1-\sum^n_{i=1}b_i\bigg)sp(x). 
\end{align}
Thus we showed that 
\begin{align}
    sp(Lv-Lu)\leq \bigg(1-\sum^n_{i=1}b_i\bigg)sp(v-u). 
\end{align}
Now from Assumption \ref{ass:unichain}, and following Theorem 8.5.3 from \cite{puterman1994markov}, it can be shown that there exists $1>\lambda>0$, such that for any $a,u,v$, 
\begin{align}
    \sum^n_{i=1} b_i \geq \lambda. 
\end{align}
Further, following Theorem 8.5.2 in \cite{puterman1994markov}, it can be shown that $L$ is a $J$-step contraction operator for some integer $J$, i.e., 
\begin{align}
    sp(L^Jv-L^Ju)\leq (1-\lambda) sp(v-u). 
\end{align}

Then, it can be shown that the relative value iteration converges to a solution of the optimal equation similar to the  relative value iteration for non-robust MDPs under the average-reward criterion (Theorem 8.5.7 in \cite{puterman1994markov}, Section 1.6.4 in\cite{sigaud2013markov}), and hence $(w_t,V_t)$ converges to a solution to \cref{eq:opt eq} as $\epsilon\to 0$. 
\end{proof}

\end{document}